\documentclass{article}

\usepackage{algorithm}
\usepackage{algorithmic}

\usepackage[T1]{fontenc}
\usepackage[hidelinks]{hyperref}
\usepackage{appendix}
\usepackage{blindtext}
\usepackage{graphicx}
\usepackage{xcolor}
\usepackage{booktabs}
\usepackage{microtype}
\usepackage{xspace}
\usepackage{multirow}
\usepackage{subcaption}
\usepackage{wrapfig}
\usepackage{tabularx}
\usepackage{pifont}

\usepackage{amsmath}
\usepackage{amsthm}
\usepackage{amssymb}
\usepackage{mathtools}
\usepackage{bm}

\usepackage[nonatbib, preprint]{neurips_2022}
\newcommand{\condition}{\ensuremath{\,|\,}}

\newcommand\x{\bm{x}}
\newcommand\s{\bm{s}}
\newcommand\z{\bm{z}}
\newcommand\expparam{\bm{\theta}}
\newcommand\priorparam{\bm{\chi}}
\newcommand\evidence{n}
\newcommand\y{y}
\newcommand{\inputdim}{D}

\newcommand\idata{i\xspace}

\newcommand\tdata{t\xspace}

\newcommand\dataix{^{(\idata)}}
\newcommand\datatx{^{(\tdata)}}

\newtheorem{theorem}{Theorem}
\newtheorem*{theorem*}{Theorem}
\newtheorem{lemma}[theorem]{Lemma}

\DeclareMathOperator{\DNIG}{\mathcal{N}\Gamma^{-1}}
\DeclareMathOperator{\DNormal}{\mathcal{N}}

\DeclareMathOperator{\real}{\mathbb{R}}

\DeclareMathOperator{\prob}{\mathbb{P}}
\DeclareMathOperator{\prior}{\mathbb{Q}}
\DeclareMathOperator{\entropy}{\mathbb{H}}
\DeclareMathOperator{\expectation}{\mathbb{E}}
\DeclareMathOperator{\variance}{Var}

\DeclareMathOperator{\policy}{\pi}

\newtheorem{desiderata}{Desideratum}[section]
\newcommand{\cmark}{\textcolor{green}{\ding{51}}}%
\newcommand{\xmark}{\textcolor{red}{\ding{55}}}%

\usepackage[utf8]{inputenc} % allow utf-8 input
\usepackage[T1]{fontenc}    % use 8-bit T1 fonts
\usepackage{hyperref}       % hyperlinks
\usepackage{url}            % simple URL typesetting
\usepackage{booktabs}       % professional-quality tables
\usepackage{amsfonts}       % blackboard math symbols
\usepackage{nicefrac}       % compact symbols for 1/2, etc.
\usepackage{microtype}      % microtypography
\usepackage{xcolor}         % colors

\usepackage[backend=bibtex,natbib=true,style=ieee,mincitenames=1,maxcitenames=2,maxbibnames=20,url=false,doi=false]{biblatex}
\AtEveryBibitem{
    \ifentrytype{inproceedings}{
         \clearlist{address}
         \clearlist{publisher}
         \clearname{editor}
         \clearlist{organization}
         \clearfield{url}  
         \clearfield{doi}  
         \clearfield{pages}  
         \clearlist{location}
         \clearlist{month}
     }{}
     \ifentrytype{article}{
         \clearlist{address}
         \clearlist{publisher}
         \clearname{editor}
         \clearlist{organization}
         \clearfield{url}  
         \clearfield{doi}  
         \clearlist{location}
     }{}
 }
\title{Disentangling Epistemic and Aleatoric Uncertainty in Reinforcement Learning}

\author{Bertrand Charpentier$^{1}$, Ransalu Senanayake$^{2}$, Mykel Kochenderfer$^{2}$, Stephan G\"unnemann$^{1}$\\
   Department of Informatics \& Munich Data Science Institute, Technical University of Munich$^{1}$\\
   Stanford Intelligent Systems Laboratory, Stanford University$^{2}$\\
   \texttt{\{charpent, guennemann\}@in.tum.de}, \texttt{\{ransalu, mykel\}@stanford.edu} \\
}

\begin{document}

\maketitle

\begin{abstract}
    
    \looseness=-1
    Characterizing \emph{aleatoric} and \emph{epistemic} uncertainty on the predicted rewards can help in building reliable reinforcement learning (RL) systems. Aleatoric uncertainty results from the irreducible environment stochasticity leading to inherently risky states and actions. Epistemic uncertainty results from the limited information accumulated during learning to make informed decisions. Characterizing aleatoric and epistemic uncertainty can be used to speed up learning in a training environment, improve generalization to similar testing environments, and flag unfamiliar behavior in anomalous testing environments. In this work, we introduce a framework for disentangling aleatoric and epistemic uncertainty in RL. \textbf{(1)} We first define four \emph{desiderata} that capture the desired behavior for aleatoric and epistemic uncertainty estimation in RL at both training and testing time. \textbf{(2)} We then present four RL \emph{models} inspired by supervised learning (i.e., Monte Carlo dropout, ensemble, deep kernel learning models, and evidential networks) to instantiate aleatoric and epistemic uncertainty. Finally, \textbf{(3)} we propose a practical \emph{evaluation} method to evaluate uncertainty estimation in model-free RL based on detection of out-of-distribution environments and generalization to perturbed environments. We present \emph{theoretical} and \emph{experimental} evidence to validate that carefully equipping model-free RL agents with supervised learning uncertainty methods can fulfill our desiderata.
    
\end{abstract}

\section{Introduction}
\label{sec:introduction}

\looseness=-1
An agent is expected to satisfy three important properties for a reliable deployment in real-world applications: \textbf{(i)} The agent should learn \emph{fast} with as few episode failures as possible. \textbf{(ii)} The agent should \emph{maintain high reward} when facing new environments similar to the training environment after deployment. \textbf{(iii)} The agent should \emph{flag anomalous environment states} when it does not know what action to take in an unknown environment. These three practically desirable properties translate into three technical properties in reinforcement learning agents. Indeed, a reinforcement learning agent should achieve high \emph{sample efficiency} at training time \cite{sample-efficient-ac}, high \emph{generalization} performance on test environments similar to the training environment \cite{epistemic-pomdp}, and high \emph{Out-Of-Distribution (OOD) detection} scores on environment unrelated to the training task \cite{ood-detection-survey, ood-automotive-perception}. 

\looseness=-1
In this paper, we argue that \emph{aleatoric} and \emph{epistemic} uncertainty are key concepts to achieve these desired practical and technical properties. The aleatoric uncertainty represents the irreducible and inherent stochasticity of the environment. Thus, an environment region with high aleatoric uncertainty is unlikely to be interesting to explore at training time because it could be uninformative (e.g. a sensor is very noisy) or dangerous (e.g. the environment has an unpredictable behavior). In contrast, the epistemic uncertainty represents the lack of information for accurate prediction. Thus, an environment region with high epistemic uncertainty is potentially promising to explore to build a better understanding of the environment (e.g., a state has unknown transition dynamics because it has never been explored).

\begin{figure*}[t]
    \centering
    \vspace{-8mm}
    \includegraphics[width=.65\linewidth]{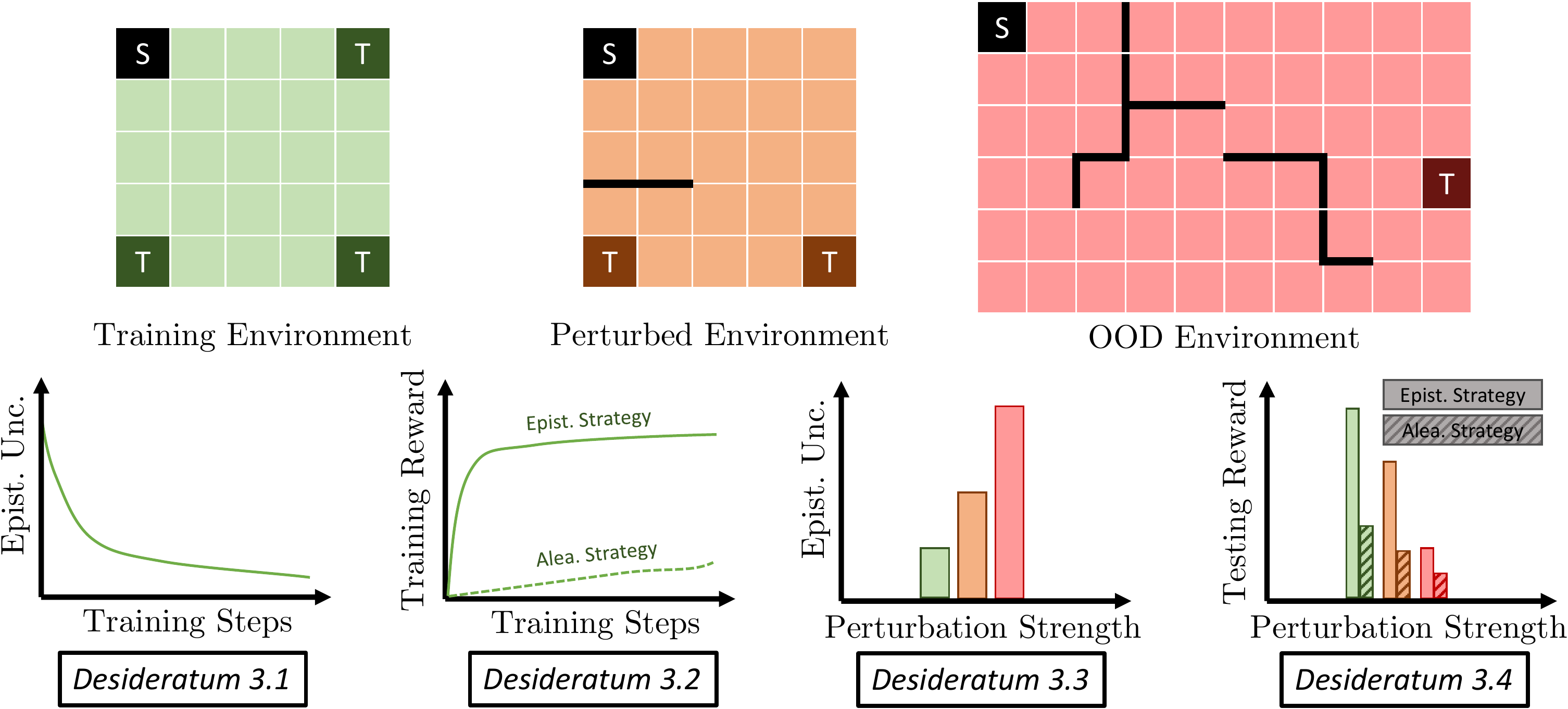}
    \caption{Overview of our proposed desiderata for uncertainty in RL (See sec.~\ref{sec:desiderata}).}
    \vspace{-3mm}
    \label{fig:diagram}
    \vspace{-4mm}
\end{figure*}

\looseness=-1
The core motivation of our work is to disentangle the properties of aleatoric and epistemic uncertainty estimates in RL to build agents with reliable performance in real-world applications. This motivation is similar to supervised learning (SL) where previous works defined \emph{desiderata}, \emph{models}, and \emph{evaluation} methods for aleatoric and epistemic uncertainty \cite{uncertainty-deep-learning, review-uncertainty-dl, dataset-shift, robustness-uncertainty-dirichlet}. Important examples of models using a single or multiple forward passes for uncertainty estimation in SL are MC dropout \cite{dropout}, ensemble \cite{ensembles, hyper-ensembles, batch-ensembles}, deep kernel learning \cite{simple-baseline-uncertainty, due, duq, uceloss}, and evidential networks \cite{postnet, priornet, natpn, evidential-regression}. Further, empirical evaluation of uncertainty estimates in SL focuses \emph{only} on testing time with Out-Of-Distribution (OOD) detection and generalization or detection of shifts \citep{dataset-shift, shifts-dataset}. In contrast to SL, the RL setting is more complex since it cares about the performance of uncertainty estimates at \emph{both} training and testing time.

\looseness=-1
\textbf{Our Contributions.} In this work, we propose a framework for aleatoric and epistemic uncertainty estimation in RL: \textbf{(Desiderata)} We explicitly define four desiderata for uncertainty estimation in RL at both \emph{training} and \emph{testing time} (See fig.~\ref{fig:diagram}). They cover the behavior of \emph{aleatoric} and \emph{epistemic} uncertainty estimates w.r.t. the sample efficiency in the training environment and, w.r.t. the generalization performance in different testing environments. \textbf{(Models)} We carefully combine a diverse set of uncertainty estimation methods in SL (i.e. MC dropout, ensemble, deep kernel learning, and evidential networks) with Deep Q-Networks (DQN) \cite{dqn}, a ubiquitous RL model that is not equipped with uncertainty estimate by default. These combinations require a \emph{minimal} modification to the training procedure of the RL agent. We discuss \emph{theoretical} evidence on the ability of these combinations to fulfill the uncertainty desiderata. \textbf{(Evaluation)} Finally, we also propose a \emph{practical} methodology to evaluate uncertainty in RL based on OOD environments and domain shifts.

\section{Problem Setup}
\label{sec:setup}

\looseness=-1
\textbf{Uncertainty in SL.} The objective of SL is to accurately predict the output \smash{$\y\dataix$} given an input \smash{$\x\dataix$} with index \smash{$\idata$}. It differentiates between two types of uncertainty: the uncertainty on the label prediction \smash{$\y\dataix$} described by the \emph{aleatoric distribution} \smash{$\prob(\y\dataix \mid \expparam\dataix)$} with parameters \smash{$\expparam\dataix$} estimated from the input $\x\dataix$, and the uncertainty on the predicted label distribution parameters $\expparam\dataix$ described by the \emph{epistemic distribution} \smash{$\prior(\expparam\dataix \mid \priorparam\dataix)$} with parameters $\priorparam\dataix$ estimated from the input $\x\dataix$. Intuitively, the variations of the aleatoric distribution will be high when the input \smash{$\x\dataix$} does not provide discriminative information to determine the label \smash{$\y\dataix$}. The variations of the epistemic distribution will be high when the input \smash{$\x\dataix$} does not provide enough information to determine the label distribution \smash{$\prior(\expparam\dataix \mid \priorparam\dataix)$} described by the parameters \smash{$\priorparam\dataix$}. Thus, the aleatoric uncertainty can be measured in practice by computing the entropy, i.e \smash{$u_\text{alea}(\x\dataix)=\entropy(\prob(\y\dataix \condition \expparam\dataix))$}, or the variance, i.e. \smash{$u_\text{alea}(\x\dataix)= \variance(\prob(\y\dataix \condition \expparam\dataix))$}, of the aleatoric distribution, while the epistemic uncertainty can be measured by computing the entropy, i.e. \smash{$u_\text{epist}(\x\dataix)=\entropy(\prior(\expparam\dataix \mid \priorparam\dataix)$}, of the epistemic distribution \cite{priornet, postnet, natpn}. Further, SL also distinguishes between sampling-based which often require multiple forward passes for uncertainty estimation, and sampling-free methods which often require a single forward pass for uncertainty estimation. Sampling-based methods like MC drop-out \cite{dropout} and ensemble \cite{ensembles, hyper-ensembles, batch-ensembles} estimate uncertainty by aggregating statistics (e.g. mean and variance) from different samples which \emph{implicitly} describe the epistemic distribution \smash{$\prior(\expparam\dataix \mid \priorparam\dataix)$}. Sampling-free methods like deep kernel learning models \cite{simple-baseline-uncertainty, due, duq, uceloss} and evidential networks \cite{postnet, priornet, natpn, evidential-regression} estimate uncertainty by \emph{explicitly} parametrizing the epistemic distributions \smash{$\prior(\expparam\dataix) \mid \priorparam\dataix)$} with known distributions such as Normal and Normal Inverse-Gamma (NIG) distributions, thus enabling efficient and closed-form computation of the distribution statistics (incl. mean, variance or entropy).

\looseness=-1
\textbf{Uncertainty in RL.} We consider the task of learning RL policies interacting with an environment with fully observed states at every time step $t$. The environment is described by a Markov Decision Process (MDP) $(S, A, r, T, \rho, \gamma)$ where $S$ is the state space, $A$ is the action space, \smash{$r(\s\datatx, a\datatx)$} is the reward associated to the action \smash{$a\datatx$} and state \smash{$\s\datatx$}, \smash{$T(\s^{(t+1)}|\s\datatx, a\datatx)$} is the transition probability, \smash{$\rho(\s^{(0)})$} is the initial state distribution, and $\gamma$ is the discount factor. Given the current state \smash{$\s\datatx$}, our goal is to learn a policy predicting the action \smash{$a\datatx$} leading to the highest reward \smash{$\y\datatx=r(\s\datatx, a\datatx)$} in addition to the aleatoric uncertainty \smash{$u_\text{alea}(\s\datatx, a\datatx)$} and the epistemic uncertainty \smash{$u_\text{epist}(\s\datatx, a\datatx)$} on the predicted reward. Similar to SL, the aleatoric and epistemic distributions can be instantiated with \smash{$\prob(\y\datatx \mid \expparam\datatx)$} and \smash{$\prior(\expparam\datatx) \mid \priorparam\datatx)$} where the predicted value is the future reward i.e. \smash{$\y\datatx=r(\s\datatx, a\datatx)$}. Intuitively, the variation of the aleatoric distribution will be high when the current state $\s\datatx$ and action $a\datatx$ only contains noisy information to determine the future reward \smash{$\y\datatx=r(\s\datatx, a\datatx)$} while the epistemic uncertainty will be high when the current state \smash{$\s\datatx$} and action \smash{$a\datatx$} does not provide enough information according to the model to determine the reward distribution \smash{$\prior(\expparam\datatx \mid \priorparam\datatx)$} described by the parameters \smash{$\priorparam\datatx$}. Finally, we consider three action selection strategies which is a crucial choice for exploration and generalization in RL: the \emph{epsilon-greedy} strategy \cite{epsilon-greedy} which selects the action with the highest predicted reward with probability $1-\epsilon$ and samples a random action otherwise, the \emph{sampling-aleatoric} strategy which takes the action with the highest predicted reward based on one aleatoric distribution sample i.e. \smash{$a\datatx = \max_{a} \y\datatx$} where \smash{$\y\datatx \sim \prob(\y\datatx | \expparam\datatx)$}, or the \emph{sampling-epistemic} strategy which takes the action with the highest predicted reward based on one epistemic distribution sample i.e. \smash{$a\datatx = \max_{a} \expectation_{\prob(\y\datatx | \expparam\datatx)}[\y\datatx]$} where \smash{$\expparam\datatx \sim \prior(\expparam\datatx | \priorparam\datatx)$}. The sampling-epistemic strategy corresponds to the Thompson sampling strategy\cite{thompson-sampling}.

\section{Desiderata for Uncertainty Quantification in RL}
\label{sec:desiderata}

\looseness=-1
In this section, we \emph{explicitly} define four intuitive and general desiderata that capture the desired behavior for uncertainty estimates in our RL setup. The desiderata cover \emph{aleatoric} and \emph{epistemic} uncertainty at both \emph{training} and \emph{testing} time. The first distinction differentiates between aleatoric and epistemic uncertainty which are commonly used concepts in SL \cite{uncertainty-deep-learning, priornet, natpn}. In contrast, RL mostly focuses on measuring aleatoric uncertainty with risk-sensitive policy or distributional RL \cite{distributional-rl-prespective, distributional-rl, iqn}. The second distinction differentiates between training and testing time relevant to sample efficiency and generalization in RL. In contrast, SL mostly focuses on testing time performance.

\looseness=-1
\textbf{Training Time.} We describe the desired behavior of uncertainty estimates at training time. First, we describe the desired uncertainty behavior when observing more samples of the training environment.
\begin{desiderata}
    \label{ax:training_state}
    An agent training longer on states sampled from one specific environment should become more epistemically confident when predicting actions on states sampled from the same specific environment.
\end{desiderata}
\vspace{-2mm}
Intuitively, an agent observing more samples from the same environment distribution should accumulate more knowledge through time thus being more epistemically certain. In practice, des.~\ref{ax:training_state} expresses that the epistemic uncertainty estimates should reflect the accumulated knowledge, and thus the convergence, of the agent during training. We test des.~\ref{ax:training_state} in the experiments (see sec.~\ref{sec:experiments}) by tracking the epistemic uncertainty at training time. Second, we describe the behavior of the total reward when selecting actions based on uncertainty estimates at training time.
\begin{desiderata}
    \label{ax:training_strategy}
    All else being equal, an agent selecting actions with the sampling-aleatoric strategy at training time should achieve lower sample efficiency than an agent selecting actions with the sampling-epistemic strategy.
\end{desiderata}
Intuitively, an agent exploring states with more (irreducible) aleatoric uncertainty would gain less knowledge about the environment dynamic than an agent exploring states with high epistemic uncertainty where the agent lacks knowledge. However, there is an important trade-off between over- or under-exploring epistemically uncertain actions which could lead to lower sample efficiency. The sampling-epistemic strategy, which corresponds to Thompson sampling \cite{thompson-sampling}, mitigates this exploration-exploitation problem by sampling action w.r.t. the epistemic distribution. Thompson sampling has already \emph{empirically} demonstrated high sample efficiency in deep RL problems \cite{dropout} and provably achieve low regret in many decision-making problems like multi-arms Bandit \cite{thompson-sampling-mab, thompson-sampling-information}. In practice, des.~\ref{ax:training_strategy} suggests that an agent should use the sampling-epistemic strategy for a better exploration-exploitation trade-off. We test des.~\ref{ax:training_strategy} in the experiments (see sec.~\ref{sec:experiments}) by comparing the sample efficiency of the sampling-aleatoric and the sampling-epistemic strategies during training.

\looseness=-1
\textbf{Testing Time.} We describe the desired behavior of uncertainty estimates at testing time. First, we describe the desired uncertainty behavior when observing samples from an environment different from the training environment.
\begin{desiderata}
    \label{ax:testing_state}
    At testing time, epistemic uncertainty should be greater in environments that are very different from the original training environments.
\end{desiderata}
\vspace{-2mm}
The environment difference could be measured with different distances depending on the task or application requirements \cite{domain-shifts-rl}. Intuitively, an agent should be less confident when observing new states at test time that were not used to collect knowledge at training time. In practice, des.~\ref{ax:testing_state} suggests that an agent should be able to use epistemic uncertainty estimates to detect states which are abnormal compared to the states observed during training. We test des.~\ref{ax:testing_state} in the experiments (see sec.~\ref{sec:experiments}) by comparing the epistemic uncertainty of the training environment against the uncertainty in noisy environments at testing time. Noisy environments include environments with completely random states and thus irrelevant to the training task, and environments with different strengths of perturbation on the original states, actions, or transition dynamics. Second, we describe the behavior of the total reward when selecting at testing time actions based on uncertainty aleatoric or epistemic uncertainty estimates.
\begin{desiderata}
    \label{ax:testing_strategy}
    All else being equal, an agent sampling actions from the epistemic uncertainty at training and testing time should generalize better at testing time than an agent sampling actions from the aleatoric uncertainty.
\end{desiderata}
\vspace{-2mm}
Intuitively, an agent exploring more epistemically uncertain states at training time would collect more knowledge about the environment, thus generalizing to more states at testing time. Further, since the environment dynamic is not directly observed, an agent should account for the epistemic certainty on the current state to take actions that generalize better at testing time. In particular, it has been shown that the Bayes-optimal Markovian policy at testing time is stochastic in general due to the partially observed MDP dynamic sometimes called epistemic POMDP \cite{epistemic-pomdp}. In practice, des.~\ref{ax:testing_strategy} suggests that an agent should use the sampling-epistemic strategy for more robust generalization performance. We test des.~\ref{ax:testing_strategy} in the experiments (see sec.~\ref{sec:experiments}) by comparing the reward obtained by the sampling-aleatoric and the sampling-epistemic strategies at testing time. The two latter desiderata \ref{ax:testing_state} and \ref{ax:testing_strategy} express an important trade-off between assigning high uncertainty and generalizing to new test environments. Since an agent cannot generalize to all new environment because of the No Free Lunch Theorem \cite{no-free-lunch-theorem-optimization}, an agent should assign higher uncertainty to environments where it does not generalize. We jointly test des.~\ref{ax:testing_state} and des.~\ref{ax:testing_strategy} in the experiments (see sec.~\ref{sec:experiments}) by tracking the reward and the uncertainty estimates in test environments with different perturbation strengths.

\section{Models for Uncertainty Quantification in RL}
\label{sec:models}

\looseness=-1
Model-free RL agents commonly rely on learning the expected return \smash{$Q^{\policy}(\s^{(t)}, a^{(t)})$} associated with taking action \smash{$a^{(t)}$} in state \smash{$\s^{(t)}$} and then following a policy $\policy$. It is defined by the Bellman equation:
% \begin{equation}
% \label{eq:bellman_equation}
    $Q^{\policy}(\s^{(t)}, a^{(t)}) = r(\s^{(t)}, a^{(t)}) + \gamma \expectation_{T, \policy} [Q^{\policy}(\s^{(t+1)}, a^{(t+1)})]$.
% \end{equation}
Similarly, the optimal policy \smash{$\policy^*$} achieving the highest expected reward satisfies  the optimal Bellman equation \cite{dynamic-programming}:
\begin{equation}
\label{eq:optimal_bellman_equation}
    Q^{\policy^*}(\s^{(t)}, a^{(t)}) = r(\s^{(t)}, a^{(t)}) + \gamma \expectation_T [\max_{a^{(t+1)}} Q^{\policy^*}(\s^{(t+1)}, a^{(t+1)})]
\end{equation}
However, the exact computation of the optimal $Q$-value is often intractable for large action or state spaces. Therefore, deep RL agents like DQN \cite{dqn}, PPO \cite{ppo} and A2C \cite{a2c} 
aim at approximating the optimal $Q$-value \smash{$Q^{\policy^*}(\s^{(t)}, a^{(t)})$} with a neural network \smash{$f_\theta(\s^{(t)}, a^{(t)})$} with parameter \smash{$\bm{\theta}$}. In particular, DQN enforces eq.~\ref{eq:optimal_bellman_equation} by minimizing the squared temporal difference (TD) error \smash{$\|r(\s^{(t)}, a^{(t)}) + \gamma \max_{a^{(t+1)}} f_{\bm{\theta}'}(\s^{(t+1)}, a^{(t+1)}) - f_{\bm{\theta}}(\s^{(t)}, a^{(t)})\|_2$}, where \smash{$f_{\bm{\theta}}$} is the learned prediction network and \smash{$f_{\bm{\theta}'}$} is the frozen target network regularly updated with the prediction network parameters during training. The TD error minimization is similar to SL regression with a MSE loss between the prediction \smash{$\hat{\y}\datatx = f_{\bm{\theta}}(\s^{(t)}, a^{(t)})$} and the target \smash{$\y\datatx = r(\s^{(t)}, a^{(t)}) + \gamma f_{\bm{\theta}'}(\s^{(t+1)}, a^{(t+1)})$} with the key difference that the exploration strategy select the targets that will be used during training.

\looseness=-1
Model-free Deep RL agents often show important limitations for uncertainty estimation because of their neural network architecture choice. For, instance, while DQN only outputs a single scalar representing the mean $Q$-value with no uncertainty estimates, PPO and A2C policies parameterized with standard ReLU networks would provably produce overconfident predictions for extreme input states \cite{overconfident-relu}. In this work, we focus on equipping the widely used DQN RL agent with reliable uncertainty estimates. To this end, we combine DQN with four SL architectures for uncertainty estimation (incl. MC dropout, ensemble, deep kernel learning, and evidential networks) covering a diverse range of sampling-based and sampling-free methods. These four DQN combinations allow to instantiate both aleatoric and epistemic uncertainty with minimal modifications to the training procedure.  We provide a summary of the uncertainty properties of these models in Tab.~\ref{tab:summary_models}.

\begin{table*}[ht!]
    \vspace{-3mm}
	\caption{Summary of the uncertainty properties of the models.}
	\label{tab:summary_models}
	\vspace{-2mm}
	\centering
	\resizebox{.89\textwidth}{!}{%
\begin{tabular}{@{}lcccc@{}}
\toprule
 & DropOut & Ensemble  & Deep Kernel Learning  & Evidential Networks \\
\midrule
Uncertainty concentration (Des.~\ref{ax:training_state}) & \xmark & \xmark & \xmark & \cmark\\
\midrule
Alea. vs epist. sampling at training time (Des.~\ref{ax:training_strategy})  & \cmark & \cmark & \xmark & \cmark\\
\midrule
OOD detection (Des.~\ref{ax:testing_state}) & \xmark & \xmark & \cmark & \cmark\\
\midrule
Alea. vs epist. sampling at testing time (Des.~\ref{ax:training_strategy}) & \cmark & \cmark & \xmark & \cmark\\
\bottomrule
\end{tabular}}
	\vspace{-3mm}
\end{table*}

\looseness=-1
\textbf{MC Dropout.} DQN is combined with MC Dropout \cite{dropout} in three steps: \textbf{(1)} it samples $K$ independent set of model parameters \smash{$\{\bm{\theta}_k\}_{k=1}^K$} by dropping activations with probability $p$, \textbf{(2)} it performs $K$ forward passes \smash{$\mu_k, \sigma_k = f_{\bm{\theta}_k}(s^{(t)}, a^{(t)})$}, and \textbf{(3)} it aggregates predictions to form the mean prediction \smash{$\mu(s^{(t)}, a^{(t)}) = \frac{1}{K}\sum_{k=1}^K \mu_k$}, the aleatoric uncertainty estimate \smash{$u_\text{alea}(s^{(t)}, a^{(t)}) = \frac{1}{K}\sum_{k=1}^K \sigma_s$}, and the epistemic uncertainty estimate \smash{$u_\text{epist}(s_t, a_t) = \frac{1}{K}\sum_{k=1}^K (\mu_k - \mu(s^{(t)}, a^{(t)}))^2$}. In this case, the aleatoric distribution is Gaussian while the epistemic distribution is implicitly represented by the sampled parameters \smash{$\{\theta_k\}_{k=1}^K$}. Further, the sampling-epistemic strategy is achieved by performing one single forward pass with a single set of sampled model parameters. This is similar to the Thompson sampling strategy used by \citet{dropout}. During training, we train the neural network parameters $\theta$ by using a Gaussian negative log-likelihood loss. The combination of DQN and dropout has been shown to practically improve sample efficiency \cite{dropout}. However, dropout has multiple limitations. First, the dropout uncertainty estimates \emph{provably} do not concentrate with more observed data \cite{randomized-prior-functions}, thus potentially violating des.~\ref{ax:training_state}. Second, there is no guarantee that dropout produce meaningful uncertainty estimates for extreme input states with a finite number of samples $K$, thus potentially violating des.~\ref{ax:testing_state}. Third, dropout might be computationally expensive for large $K$ value since it would require many forward passes for uncertainty estimation.

\looseness=-1
\textbf{Ensemble.} DQN is combined with ensembles \cite{ensembles} in three steps: \textbf{(1)} it trains $K$ independent models with parameters \smash{$\{\theta_k\}_{k=1}^K$}, \textbf{(2)} it performs $K$ forward passes \smash{$\mu_k, \sigma_k = f_{\bm{\theta}_k}(s^{(t)}, a^{(t)})$}, and \textbf{(3)} it aggregates predictions to form the mean prediction \smash{$\mu(\s^{(t)}, a^{(t)}) = \frac{1}{K}\sum_{k=1}^K \mu_k$}, the aleatoric uncertainty estimate \smash{$u_\text{alea}(s^{(t)}, a^{(t)}) = \frac{1}{K}\sum_{k=1}^K \sigma_s$}, and the epistemic uncertainty estimate \smash{$u_\text{epist}(s_t, a_t) = \frac{1}{K}\sum_{k=1}^K (\mu_k - \mu(s^{(t)}, a^{(t)}))^2$}. In this case, the aleatoric distribution is Gaussian while the epistemic distribution is implicitly represented by the parameters of the $K$ networks \smash{$\{\bm{\theta}_k\}_{k=1}^K$}. Further, the sampling-epistemic strategy is achieved by performing one single forward pass with one randomly selected network. This is similar to the Thompson sampling strategy used by \citet{bootstrapped-dqn}. We train the $K$ independent neural network parameters \smash{$\theta_k$} with a Gaussian negative log-likelihood loss. However, ensemble has multiple limitations. First, while the combination of DQN with bootstrapped ensemble and prior functions has been \emph{empirically} shown to improve learning for complex tasks with sparse rewards \cite{bootstrapped-dqn, randomized-prior-functions}, there is no explicit theoretical or empirical evidence that their uncertainty estimates concentrate with more observed data. Second, there is no guarantee that ensembles produce meaningful uncertainty estimates for extreme input states with a finite number of samples $K$, thus potentially violating des.~\ref{ax:testing_state}. Third, ensemble is computationally expensive for large $K$ value since it would require many forward passes and many neural networks. 

\looseness=-1
\textbf{Deep Kernel Learning.} DQN is combined with deep kernel learning \cite{due} in three steps: \textbf{(1)} it predicts one latent representation of each input state i.e. \smash{$\z^{(t)} = f_{\bm{\theta}}(\s^{(t)})$}, and \textbf{(2)} one Gaussian Process per action $a$ defined from a fixed set of $K$ learnable inducing points \smash{$\{\bm{\phi}_{a,k}\}_{k=1}^{K}$} and a predefined positive definite kernel \smash{$\kappa(\cdot, \cdot)$} predicts the mean \smash{$\mu(\s^{(t)}, a)$} and the variance \smash{$\sigma(\s^{(t)}, a)$} of a Gaussian distribution. We train the neural network parameters $\bm{\theta}$ and the inducing points \smash{$\{\bm{\phi}_{a,k}\}_{k=1}^{K}$} jointly with a variational ELBO loss similarly to \cite{due}. In this case, the epistemic distribution is Gaussian \cite{simple-baseline-uncertainty}, i.e. \smash{$u_\text{epist}(s_t, a_t) = \entropy(\DNormal(\mu(\s^{(t)}, a^{(t)}), \sigma(\s^{(t)}, a^{(t)})))$}. Indeed, we show \emph{theoretically} that epistemic uncertainty increases far from training data (see app.~\ref{app:proofs}). Thus, the combination of DQN and deep kernel learning does not suffer from arbitrary uncertainty estimates for extreme input states contrary to ReLU networks \cite{overconfident-relu}. However, one of the limitation of deep kernel learning is that it does not disentangle \emph{aleatoric} and \emph{epistemic} uncertainty. This is similar to deep kernel learning methods in SL \cite{simple-baseline-uncertainty, due, duq}.

\looseness=-1
\textbf{Evidential Networks.} The combination of DQN and the posterior networks \cite{natpn, postnet}  which belong to the class of evidential networks consists in three steps: \textbf{(1)} an encoder $ f_{\bm{\theta}}$ predicts one latent representation of each input state i.e. \smash{$\z^{(t)} = f_{\bm{\theta}}(\s^{(t)})$}, \textbf{(2)} one normalizing flow density estimator \smash{$\prob(. \condition \bm{\omega}_{a})$} and one linear decoder \smash{$g_{\bm{\psi}_{a}}$} per action $a$ predict a Normal Inverse-Gamma distribution \smash{$\prior(\priorparam(\s^{(t)}, a), \evidence(\s^{(t)}, a))$} with parameters \smash{$\priorparam(\s^{(t)}, a) = g_{\bm{\psi}_{a}}(\z^{(t)}, a)$} and \smash{$ \evidence(\s^{(t)}, a) \propto \prob(\z\dataix \condition \bm{\omega}_{a})$}, and \textbf{(3)} it computes the posterior parameters $\priorparam^\text{post}(\s^{(t)}, a) = \frac{\evidence^\text{prior}\priorparam^\text{prior} + \evidence(\s^{(t)}, a)) \priorparam(\s^{(t)}, a)}{\evidence^\text{prior} + \evidence(\s^{(t)}, a))}, \evidence^\text{post}(\s^{(t)}, a)) = \evidence^\text{prior} + \evidence(\s^{(t)}, a)$ where the prior parameters are chosen to enforce high entropy for the prior distribution e.g. \smash{${\priorparam^\text{prior}=(0, 100)^T}, \evidence^\text{prior}=1$} \cite{natpn}. In this case, the epistemic distribution is a Normal Inverse-Gamma distribution and the aleatoric distribution is a Normal distribution \cite{natpn}. We train the neural network parameters $\bm{\theta}$ and $\bm{\psi}$ and the normalizing flow parameters $\bm{\omega}$ jointly with the MSE loss. The entropy of the conjugate prior distribution represents the epistemic uncertainty , i.e. \smash{$u_\text{epist}(s\datatx, a\datatx) = \entropy(\DNIG(\priorparam(\s\datatx, a\datatx), \evidence(\s\datatx, a\datatx)))$}. The entropy of the likelihood distribution represents the aleatoric uncertainty, i.e. \smash{$u_\text{alea}(s\datatx, a\datatx) = \entropy(\DNormal(\mu(\s\datatx, a\datatx), \sigma(\s\datatx, a\datatx)))$}. Indeed, it has been showed \emph{theoretically} that epistemic uncertainty increases far from training data (see app.~\ref{app:proofs}) \cite{natpn}. Thus, the combination of DQN and posterior networks does not suffer from arbitrary uncertainty estimates for extreme input states contrary to ReLU networks \cite{overconfident-relu}.

\section{Evaluation of Uncertainty Quantification in RL}
\label{sec:experiments}

\looseness=-1
In this section, we provide an extensive evaluation of uncertainty estimation for model-free RL. It compares four uncertainty estimation methods for model-free RL in three environments. First, we evaluate the uncertainty predictions at \emph{training time} to assess the uncertainty concentration (des.~\ref{ax:training_state}) and the sample efficiency of the uncertainty-guided exploration-exploitation strategy (des.~\ref{ax:training_strategy}). Second, we evaluate the uncertainty estimates at \emph{testing time} to assess the OOD detection performances (des.~\ref{ax:testing_state}) and the generalization performances of the uncertainty-guided decision strategy (des.~\ref{ax:testing_strategy}). In particular, we evaluate the trade-off between the generalization performance and the detection performance in new perturbed test environments.

\looseness=-1
\textbf{Models.} We consider the four uncertainty models MC Dropout \textbf{(DropOut)}, \textbf{Ensemble}, deep kernel learning \textbf{(DKL)} and the evidential model based on Posterior Networks \textbf{(PostNet)} combined with the DQN RL policy (see sec.~\ref{sec:models}). In this work, we focus on DQN \cite{dqn} since it is a widely used model-free RL agent which does not provide any uncertainty estimates by default. All models use the same encoder architecture and DQN hyper-parameters. We performed a grid search over all hyper-parameters. We compute the mean and standard error of the mean over 5 seeds. Further details are given in app.~\ref{app:models-details}.

\looseness=-1
\textbf{Environments.} We used three training environments \textbf{CartPole} \cite{cartpole}, \textbf{Acrobot} \cite{acrobot1, acrobot2} and \textbf{LunarLander} \cite{lunarlander1, lunarlander2} from the Open AI gym environments \cite{gym}. \citet{assessing-generalization-rl} also used similar environment to assess generalization in RL. We focus on these environments since they turn out to be already challenging settings for the uncertainty methods and the sampling strategies. Some methods and strategies are indeed already unable to achieve high performance for sample efficiency, generalization, and OOD detection. We provide further details on the environments in app.~\ref{app:environments-details}. \textit{\underline{OOD environments:}} The states, actions, and transition dynamics of the OOD environments should not be relevant to the original training environment task, thus being a reasonable failure mode. To this end, the input state is composed of Gaussian noise at every time step independently of the previous actions. \textit{\underline{Perturbed environments:}} These environments are perturbed versions of the original training environment with different perturbation strengths. We separately perturb the \emph{state} space, the \emph{action} space, and the \emph{transition} dynamics with different strengths of Gaussian or uniform noises. These perturbations follow the MDP structure of the environment as proposed by the formal framework for domain shifts presented in \cite{domain-shifts-rl}. We did not consider perturbation on the initial state only, which would be a weaker version of the state perturbations, and perturbations on the reward function which would not affect the model at testing time. Further details are given in app.~\ref{app:environments-details}.

\looseness=-1
\textbf{Training Time.} First, we compare the sample efficiency and the uncertainty predictions of the four uncertainty methods using the epsilon-greedy exploration-exploitation strategy at training time. We normalize the epistemic uncertainty in $[0, 1]$ with min-max normalization to compute the relative epistemic uncertainty. It allows us to easily compare the trend of the epistemic uncertainty of all models. We show the key results for Cartpole in fig.~\ref{fig:model-training-testing-performance-cartpole-main} and the detailed results for CartPole, Acrobot and LunarLander in fig.~\ref{fig:model-training-performance-cartpole}, \ref{fig:model-training-performance-acrobot}, \ref{fig:model-training-performance-lunarlander} in app.~\ref{app:additional-experiments}. We observe that all methods achieve similar sample efficiency. Ensemble with epsilon-greedy strategy struggles to maintain high reward on CartPole. This can be intuitively explained by the under-exploration of the epsilon-greedy strategy as also observed by \citet{randomized-prior-functions, dropout}. Further, we observe that only the epistemic uncertainty estimates of the combination of DQN and PostNet decreases during training. Thus, PostNet \emph{empirically} validates des.~\ref{ax:training_state}. In contrast, the epistemic uncertainty estimates of other methods increase or do not converge. This corroborates with the findings of \cite{randomized-prior-functions} which \emph{theoretically} shows that the uncertainty estimates of dropout and ensemble might not converge even on simple tasks.

\begin{figure}
    \centering
    \vspace{-5mm}
        \begin{subfigure}{.5\textwidth}
        \includegraphics[width=\textwidth]{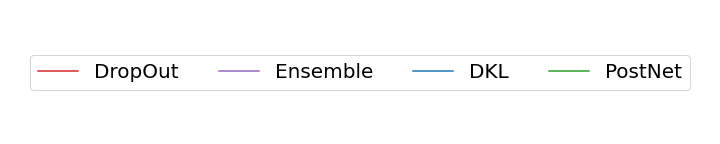}
    \end{subfigure}
    \vspace{-5mm}
    
    \begin{subfigure}{.245\textwidth}
        \includegraphics[width=\textwidth]{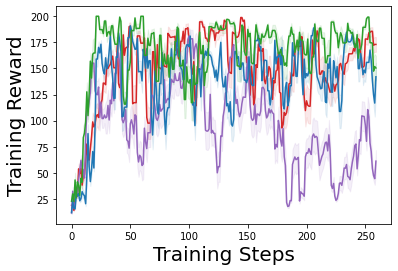}
        \vspace{-5mm}
        \caption{}
         \label{fig:model-training-reward-cartpole}
   \end{subfigure}
    \begin{subfigure}{.245\textwidth}
        \includegraphics[width=\textwidth]{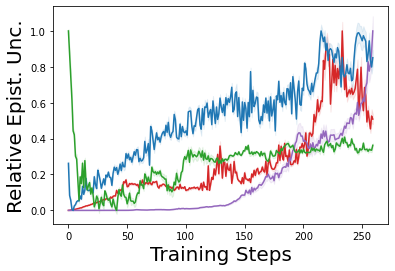}
        \vspace{-5mm}
        \caption{}
        \label{fig:model-training-uncertainty-cartpole}
    \end{subfigure}
    \begin{subfigure}{.245\textwidth}
        \includegraphics[width=\textwidth]{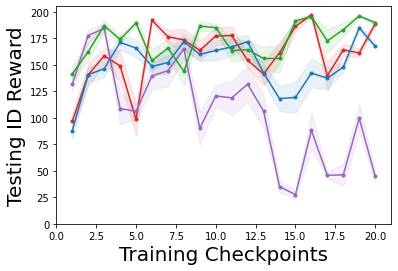}
        \vspace{-5mm}
        \caption{}
        \label{fig:model-testing-reward-cartpole}
    \end{subfigure}
    \begin{subfigure}{.245\textwidth}
        \includegraphics[width=\textwidth]{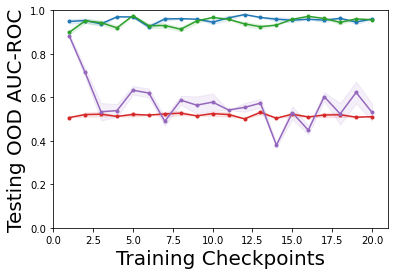}
        \vspace{-5mm}
        \caption{}
        \label{fig:model-testing-ood-cartpole}
    \end{subfigure}
    \vspace{-3mm}
    \caption{Comparison of the training performance (\ref{fig:model-training-reward-cartpole},  \ref{fig:model-training-uncertainty-cartpole}) and testing performance (\ref{fig:model-testing-reward-cartpole},  \ref{fig:model-testing-ood-cartpole}) of the four uncertainty methods using epsilon-greedy strategies on CartPole. Ideally, an uncertainty aware-model in RL should achieve high reward with few training samples at training and testing time, a decreasing epistemic uncertainty at training time and high OOD detection scores at testing time.}
    \label{fig:model-training-testing-performance-cartpole-main}
    \vspace{-5mm}
\end{figure}

\looseness=-1
Second, we compare the sample and episode efficiency of the \emph{sampling-epistemic} and the \emph{sampling-aleatoric} strategies for each model during training. We show the results for Acrobot in fig.~\ref{fig:strategy-training-performance-acrobot-main}, and additional results for CartPole and LunarLander in fig.~\ref{fig:strategy-training-performance-cartpole}, \ref{fig:strategy-training-performance-lunarlander} in app.~\ref{app:additional-experiments}. We observe that all models achieve high rewards by using the sampling-epistemic strategy. In particular, we observed that Ensemble with epistemic sampling achieves more stable rewards than with epsilon-greedy which aligns with observations in \cite{bootstrapped-dqn}. Further, we observe that all models instantiating both aleatoric and epistemic uncertainty achieve significantly better sample efficiency with sampling-epistemic than sampling-aleatoric. Contrary to the sampling-epistemic strategy, the sampling-aleatoric strategy intuitively fails at visiting new under-explored states/actions, thus achieving low and unstable rewards. Hence, Drop-Out, Ensemble and PostNet \emph{empirically} validate des.~\ref{ax:training_strategy}. Thus, disentangling aleatoric and epistemic uncertainty can speed learning in a training environment. Further, the sampling-epistemic strategy requires fewer finished episodes on CartPole and LunarLander. This represents a more reliable training for these two environments since each finished episode translates into a failure and a restart of the systems (see app.~\ref{app:additional-experiments}).

\begin{figure}
    \centering
    \vspace{-3mm}
    \begin{subfigure}{.45\textwidth}
        \includegraphics[width=\textwidth]{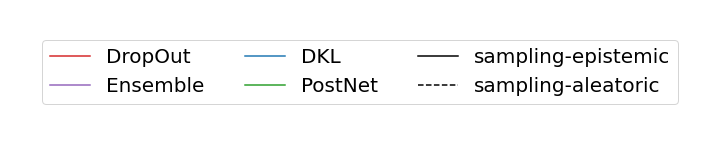}
    \end{subfigure}
    \vspace{-3mm}
    
    \begin{subfigure}{.245\textwidth}
        \includegraphics[width=\textwidth]{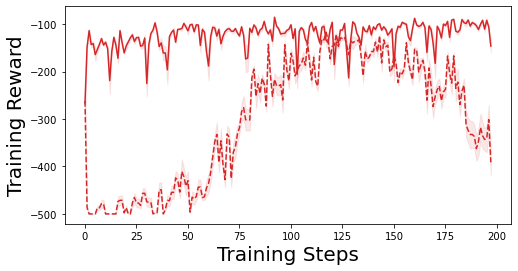}
    \end{subfigure}
    \begin{subfigure}{.245\textwidth}
        \includegraphics[width=\textwidth]{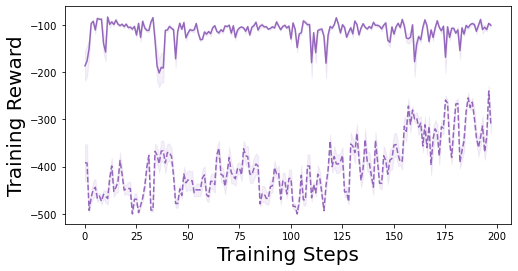}
    \end{subfigure}
    \begin{subfigure}{.245\textwidth}
        \includegraphics[width=\textwidth]{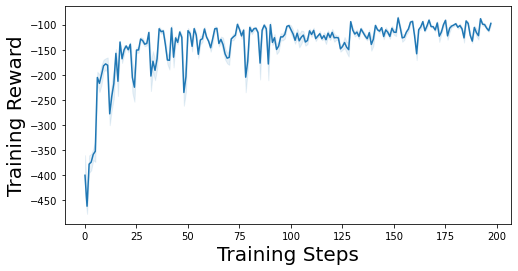}
    \end{subfigure}
    \begin{subfigure}{.245\textwidth}
        \includegraphics[width=\textwidth]{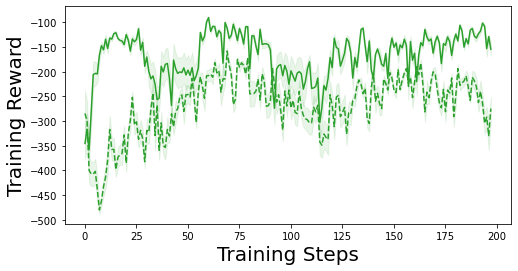}
    \end{subfigure}
    
    % \begin{subfigure}{.245\textwidth}
    %     \includegraphics[width=\textwidth]{resources/acrobot-n_finished_training_episodes-dropout-training-strategy.png}
    % \end{subfigure}
    % \begin{subfigure}{.245\textwidth}
    %     \includegraphics[width=\textwidth]{resources/acrobot-n_finished_training_episodes-ensemble-training-strategy.png}
    % \end{subfigure}
    % \begin{subfigure}{.245\textwidth}
    %     \includegraphics[width=\textwidth]{resources/acrobot-n_finished_training_episodes-dkl-training-strategy.png}
    % \end{subfigure}
    % \begin{subfigure}{.245\textwidth}
    %     \includegraphics[width=\textwidth]{resources/acrobot-n_finished_training_episodes-postnet-training-strategy.png}
    % \end{subfigure}
    \vspace{-3mm}
    \caption{Comparison of the training performance on Acrobot. The four uncertainty methods use the sampling-aleatoric or the sampling-epistemic at training time. Ideally, an uncertainty aware-model should high reward with few samples.}
    \label{fig:strategy-training-performance-acrobot-main}
    \vspace{-5mm}
\end{figure}
% \input{figures/strategy-training-performance-acrobot}
%\input{figures/strategy-training-performance-lunarlander}

%%% Testing time

\looseness=-1
\textbf{Testing Time.} In this section, we save $20$ models during training to evaluate their performance at testing time. The testing performance can be viewed as the model performance after deployment. First, we evaluate the testing in-distribution (ID) reward in the training environment and the out-of-distribution (OOD) detection performance against the OOD environment composed of fully noisy states. All the methods used the same epsilon-greedy strategy at training time and the action lead to the highest predicted expected return at testing time. The OOD detection performance is measured by comparing the predicted epistemic uncertainty of the states/actions of $10$ episodes with the area under the receiver operating characteristic curve (AUC-ROC). We show the results for CartPole in fig.~\ref{fig:model-training-testing-performance-cartpole-main}, and additional results for Acrobot and LunarLander in fig.~\ref{fig:model-testing-performance-acrobot} and fig.~\ref{fig:model-testing-performance-lunarlander} in app.~\ref{app:additional-experiments}. We observe that DKL and PostNet achieve very high OOD detection scores compared to DropOut and Ensemble. These \emph{empirical} results align with the \emph{theoretical} results stating that DKL and PostNet should assign high uncertainty to states very different from states observed during training. Thus, DKL and PostNet validate des.~\ref{ax:testing_state}. In particular, DKL and PostNet can reliably equip DQN with epistemic uncertainty estimates which can be used to flag anomalous OOD states. In contrast, DropOut and Ensemble achieve poor OOD detection scores. This aligns with \citet{randomized-prior-functions, natpn} which shows on multiple experiments that the uncertainty estimates assigned to OOD inputs by DropOut and Ensemble are not significantly smaller than the uncertainty estimates assigned to inputs close to training data.

\looseness=-1
Second, we compare the testing in-distribution (ID) reward and the out-of-distribution (OOD) detection performance when models use the \emph{sampling-aleatoric} and \emph{sampling-epistemic} strategies at \emph{both} training and testing time. We show the results for the testing reward and the OOD detection scores on the LunarLander in fig.~\ref{fig:strategy-testing-performance-lunarlander}, and additional results on the CartPole and the Acrobot environments in fig.~\ref{fig:strategy-testing-performance-cartpole} and fig.~\ref{fig:strategy-testing-performance-acrobot} in app.~\ref{app:additional-experiments}. We observe that the sampling-epistemic strategy achieves significantly better rewards than the sampling-aleatoric for almost any checkpointed models during training. Thus, all models \emph{empirically} satisfy des.~\ref{ax:testing_state}. These empirical results underline the need to disentangle both aleatoric and epistemic uncertainty for high reward performance at testing time.

\begin{figure}
    \centering
        \vspace{-6mm}
    \begin{subfigure}{.45\textwidth}
        \includegraphics[width=\textwidth]{resources/sampling-legend.png}
    \end{subfigure}
    \vspace{-3mm}
    
    \begin{subfigure}{.245\textwidth}
        \includegraphics[width=\textwidth]{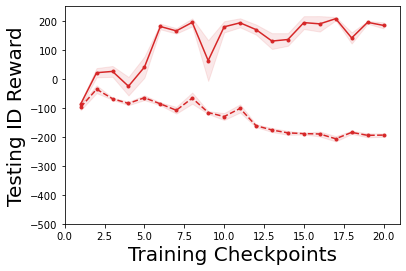}
    \end{subfigure}
    \begin{subfigure}{.245\textwidth}
        \includegraphics[width=\textwidth]{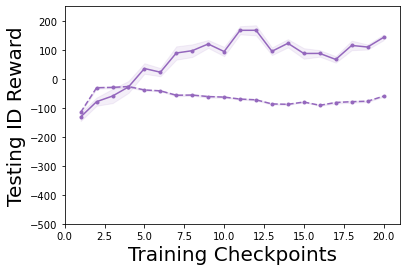}
    \end{subfigure}
    \begin{subfigure}{.245\textwidth}
        \includegraphics[width=\textwidth]{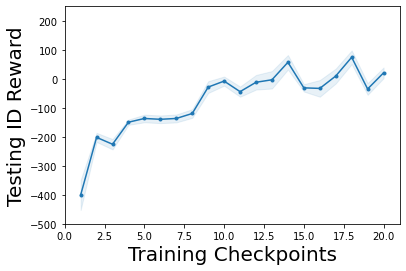}
    \end{subfigure}
    \begin{subfigure}{.245\textwidth}
        \includegraphics[width=\textwidth]{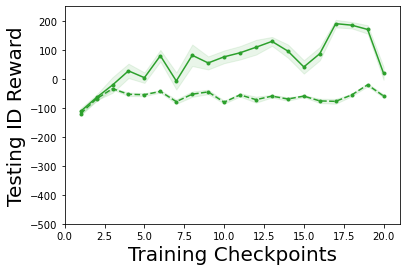}
    \end{subfigure}
    
    \begin{subfigure}{.245\textwidth}
        \includegraphics[width=\textwidth]{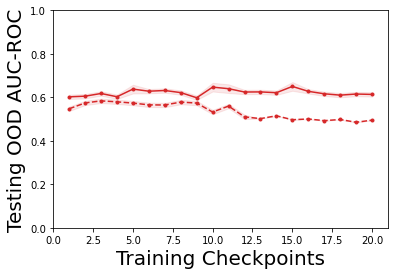}
    \end{subfigure}
    \begin{subfigure}{.245\textwidth}
        \includegraphics[width=\textwidth]{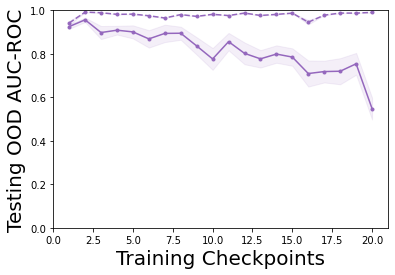}
    \end{subfigure}
    \begin{subfigure}{.245\textwidth}
        \includegraphics[width=\textwidth]{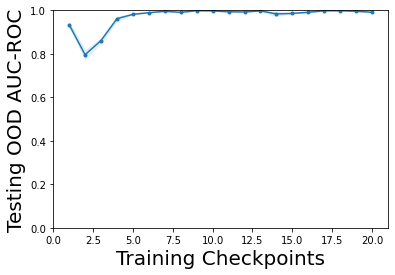}
    \end{subfigure}
    \begin{subfigure}{.245\textwidth}
        \includegraphics[width=\textwidth]{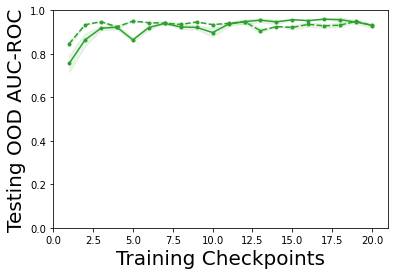}
    \end{subfigure}
        \vspace{-3mm}
    \caption{Comparison of the testing reward and OOD performance on LunarLander. The four uncertainty methods use the sampling-aleatoric or sampling-epistemic strategies at both training and testing time. Ideally, an uncertainty aware-model should achieve high testing reward and high OOD AUC-ROC detection score.}
    \label{fig:strategy-testing-performance-lunarlander}
        \vspace{-6mm}
\end{figure}

\looseness=-1
Third, we compare the \emph{sampling-epistemic} and the \emph{sampling-aleatoric} strategies for each model at testing time. All models use the same epsilon-greedy strategy at training time. We show the key results for the testing reward and the testing epistemic uncertainty on Cartpole with perturbed states in fig.~\ref{fig:strategy-state-shift-testing-performance-cartpole}, and detailed results with other perturbations and environments in app.~\ref{app:additional-experiments}. We observe that stronger state and action perturbations deteriorate the reward performance of all models. This is reasonable since the input state or the output actions become more different from the training environment with stronger perturbations. Further, while the models were trained using the same epsilon-greedy strategy, we observe that the sampling-epistemic strategy generalizes significantly better to all types of perturbed environments than the sampling-aleatoric strategy. In particular, all models achieve high rewards with epistemic sampling on environments with perturbed transitions. Intuitively, sampling-aleatoric select actions with more inherent risk, while the sampling-epistemic select actions accounting for the knowledge accumulated by the agent in the training environment. The generalization capacity of the sampling-epistemic strategy aligns with \cite{epistemic-pomdp} which recast the problem of generalization in RL as solving an epistemic POMDP. Thus, differentiating between aleatoric and epistemic uncertainty can improve generalization. Finally, we observed that DKL and PostNet consistently assign higher epistemic uncertainty to environments with perturbed states which aligns with their theoretical guarantees on extreme input states. The most challenging perturbations are perturbed actions since none of the models provide guarantees for this perturbation type. Overall, DKL and PostNet reliably assign higher epistemic uncertainty to most of the perturbation types. Therefore, DKL and PostNet performs a good trade-off between generalization and detection of new perturbed environments.

\begin{figure}
    \centering
        \vspace{-6mm}
    \begin{subfigure}{.45\textwidth}
        \includegraphics[width=\textwidth]{resources/sampling-legend.png}
    \end{subfigure}
    \vspace{-3mm}
    
    \begin{subfigure}{.245\textwidth}
        \includegraphics[width=\textwidth]{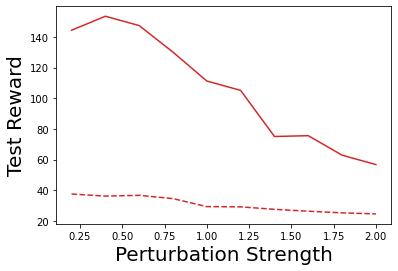}
    \end{subfigure}
    \begin{subfigure}{.245\textwidth}
        \includegraphics[width=\textwidth]{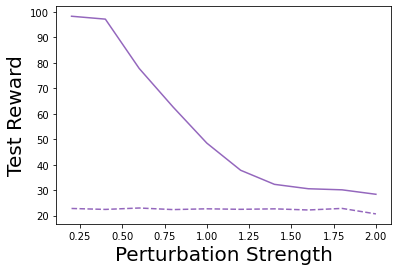}
    \end{subfigure}
    \begin{subfigure}{.245\textwidth}
        \includegraphics[width=\textwidth]{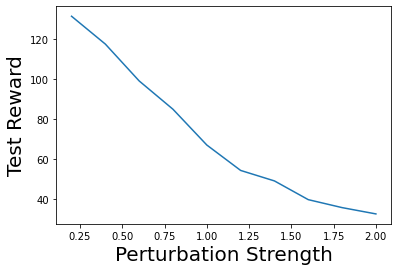}
    \end{subfigure}
    \begin{subfigure}{.245\textwidth}
        \includegraphics[width=\textwidth]{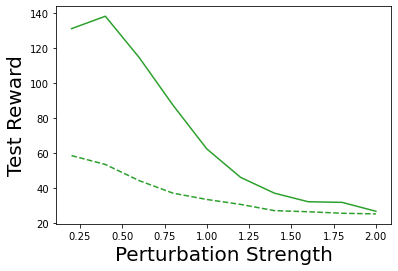}
    \end{subfigure}
    
    \begin{subfigure}{.245\textwidth}
        \includegraphics[width=\textwidth]{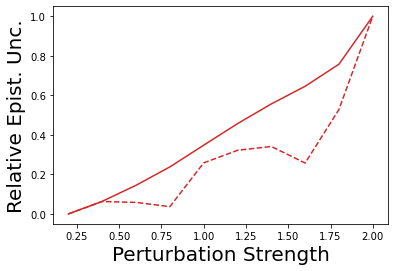}
    \end{subfigure}
    \begin{subfigure}{.245\textwidth}
        \includegraphics[width=\textwidth]{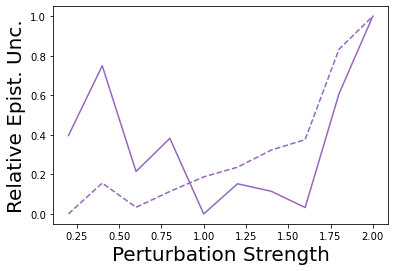}
    \end{subfigure}
    \begin{subfigure}{.245\textwidth}
        \includegraphics[width=\textwidth]{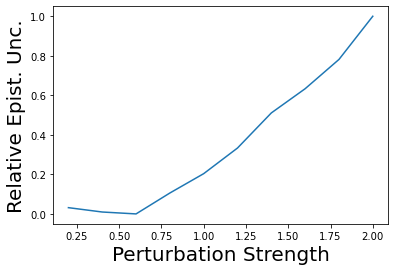}
    \end{subfigure}
    \begin{subfigure}{.245\textwidth}
        \includegraphics[width=\textwidth]{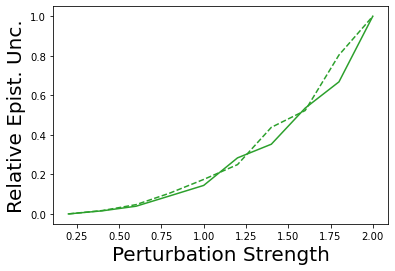}
    \end{subfigure}
        \vspace{-2mm}
    \caption{Comparison of the testing performance and the epistemic uncertainty predictions on CartPole with perturbed states. The four uncertainty methods use the epsilon-greedy strategy at training time and the sampling-aleatoric or sampling-epistemic strategy at testing time. Ideally, an uncertainty-aware model should maintain high reward while assigning higher epistemic uncertainty on more severe perturbations.}
    \label{fig:strategy-state-shift-testing-performance-cartpole}
        \vspace{-6mm}
\end{figure}
\section{Related Work}

In this section, we cover the related work for aleatoric and epistemic uncertainty estimation for RL. To this end, we review implicit \emph{desiderata} for aleatoric and epistemic uncertainty, RL \emph{methods} using uncertainty estimates and the existing \emph{evaluation} methodology to validate the quality of the uncertainty estimates in practice. We refer the reader to the survey \cite{review-uncertainty-dl} for an exhaustive overview on uncertainty estimation. 
\underline{\textit{Desiderata:}} The notion of \emph{risk} is well studied in RL and closely connected to the notion of uncertainty. Risk-sensitive RL usually aims at reducing the number of failures at training time for safer RL \cite{risk-sensitive-rl, risk-constrained-rl-percentile, risk-sensitive-mdp, safe-rl-survey}. In particular, \cite{risk-uncertainty-deep-rl, rl-risk-sample-trade-off, epistemic-risk} discuss the trade-off between risk-sensitivity and sample efficiency at training time. Further, \cite{epistemic-pomdp} aim at improving generalization at testing time within the Bayesian RL framework. In SL, the predicted uncertainty is expected to increase further from training data \citep{provable-uncertainty, natpn, bayesian-a-bit, graph-postnet}. None of these previous works give \emph{explicit} desiderata for both \emph{aleatoric} and \emph{epistemic} \emph{uncertainty} in RL at both \emph{training} and \emph{testing} \emph{time}. 
\underline{\textit{Models:}} The related work for uncertainty-aware model in RL is rich as shown in surveys on distributional and Bayesian  RL \cite{bayesian-rl, distributional-rl}. Distributional RL \cite{distributional-rl, nonparametric-return-distribution, distributional-rl-prespective, iqn} aims at learning the distribution of return which generally captures the aleatoric uncertainty \cite{information-directed-exploration-deep-rl}. Bayesian RL methods includes sampling-based methods models such as on dropout \cite{dropout, uncertainty-rl-collision-avoidance} and ensembles \cite{bootstrapped-dqn, randomized-prior-functions, safe-rl-model-uncertainty, rl-active-learning, epistemic-pomdp}. These methods are often combined with bootstrapping during training. In particular, \cite{risk-uncertainty-deep-rl} proposed to decompose aleatoric and epistemic uncertainty to the cost of multiple trained networks and \cite{decomposition-uncertatinty-bayesian-dl} decompose aleatoric and epistemic uncertainty with latent variables for model-based RL. Bayesian RL also includes Gaussian processes \citep{gp-rl, rl-gp} and more specifically the deep kernel learning method \citep{deep-rl-deep-kernel-leanring} which requires storing uncertainty estimates in the experience replay buffer during training. Unlike RL, SL includes many uncertainty methods using deep kernel learning \citep{due, duq, simple-baseline-uncertainty} and evidential network \cite{postnet, natpn, graph-postnet, priornet, regression-priornet, evidential-regression, robustness-uncertainty-dirichlet}. In contrast, we look at \emph{both} sampled-based and sampled-free uncertainty methods for \emph{aleatoric} and \emph{epistemic} uncertainty estimation with \emph{minimal} modification to the training procedure of the RL agent, thus ensuring easy adaptation of new uncertainty quantification techniques from SL to RL. 
\underline{\textit{Evaluation:}} \cite{bsuite-rl, testbed-rl} proposed to evaluate uncertainty in RL by focusing on joint predictive distributions instead of marginal distributions. Many works \cite{sample-efficient-ac, sample-efficient-rl-stochastic-ensemble, q-prop-sample-efficient, rl-fast-slow} used sample efficiency as evaluation method. Further, previous works proposed generalization benchmarks for RL \cite{generalization-rl-survey, assessing-generalization-rl, qyantifying-generalization-rl, procgen}. Finally, \cite{ood-dynamic-benchmark, benchmark-ood-detection-rl} have recently proposed benchmarks for OOD detection relevant to RL. In contrast, we propose a simple evaluation method which \emph{jointly} look at \emph{multiple} tasks relevant to real-world applications of uncertainty in RL. It covers epistemic uncertainty tracking and sample efficiency at training time, and generalization and OOD detection at testing time. In particular, we evaluate the trade-off between OOD generalization \cite{ood-generalization-survey} and OOD detection \cite{ood-detection-survey}.
\section{Limitations and Broader Impact}
\label{sec:limitations}

\looseness=-1
\underline{\textit{Desiderata:}.} Our desiderata, similar to \cite{graph-postnet, desiderata-ml-lifecycle, overview-interpretable-ml}, are designed to be application and model agnostic. In practice, the desiderata should be instantiated with formal definitions and could be customized depending on the application. \underline{\textit{Models:}} To validate the key contributions, similar to \cite{distributional-rl-prespective, iqn, bootstrapped-dqn}, we restrict our experiments to DQNs. However, the four uncertainty methods essentially modify the encoder architecture, it is possible to adapt them to other model-free RL methods such as PPO \cite{ppo} and A2C \cite{a2c}. \underline{\textit{Evaluation:}} Our approach focuses on a simple and task-diverse evaluation methodology for uncertainty estimation. Contrary to \cite{procgen}, we do not focus on scaling RL methods to more complex tasks in this paper. \underline{\textit{Broader impact:}} Our framework discusses the benefit of using uncertainty estimation to create robust and safe RL methods which corroborate with the Assessment List for Trustworthy AI \cite{trustworthy-ai}. Although, there is always a risk that this framework does not fully capture the real-world complexity, thus encouraging practitioners to proactively validate their models in the real-world.

\section{Conclusion}
\label{sec:conclusion}

\looseness=-1
We introduce a new framework to characterize aleatoric and epistemic uncertainty estimation in RL. It includes four explicit desiderata, four RL models inspired from SL and a practical evaluation methodology. The desiderata characterize the behavior of uncertainty estimates at both training and testing time. The models combine DQN with sampling-based and sampling-free uncertainty methods in SL without modifications of the RL agents training. We give theoretical and empirical evidence that these methods can fulfil the uncertainty desiderata. The evaluation method assesses the quality of uncertainty estimates on sample efficiency, generalization and OOD detection tasks.
\section*{Acknowledgments}

The authors would like to thank Daniel Z\"ugner for the helpful discussion and comments. The authors of this work take the full responsibilities for its content. 

\printbibliography

\newpage

\newpage

%%%%%%%%%%%%%%%%%%%%%%%%%%%%%%%%%%%%%%%%%%%%%%%%%%%%%%%%%%%%

\appendix
\newpage
\onecolumn
\appendix

\section*{Appendix}

%\bc{
%\begin{itemize}
%    \item (Optional) New loss (matching distribution, time dependent ?) inspired from distributional RL)
%    \item (Optional) Calibration: Check if predicted return distribution is calibrated wrt true return.
%\end{itemize}
%}

\section{Proofs}
\label{app:proofs}

In this section, we show that deep kernel learning \cite{due} and evidential models based on Posterior Networks \cite{postnet, natpn}  are guaranteed to assign high epistemic uncertainty for inputs far from inputs observed during training under technical assumptions. In particular, the combination of DQN with deep kernel learning or evidential networks presented in sec.~\ref{sec:models} are guaranteed to assign high epistemic uncertainty for extreme input states. We assume that the encoder should use ReLU activations, which is common in deep learning, and that the rows of the linear transformations are independent, which is realistic for trained networks with no constant output \citep{overconfident-relu}.

\begin{lemma}
\label{lem:relu-regions}
\cite{understanding-nn-relu} Let $\{Q_l\}_l^{R}$ be the set of linear regions associated to the piecewise ReLU network $f_{\phi}(\x)$. For any $\x \in \real^\inputdim$, there exists $\delta^* \in \real^{+}$ and $l^*\in {1,..., R}$ such that $\delta \cdot \x \in Q_{l^*}$ for all $\delta > \delta^*$.
\end{lemma}

\begin{lemma}
\label{lem:asymptotic-latent-norm}
Let a (deep) encoder $f_{\phi}$ with piecewise ReLU activations. Let $f_{\phi}(\x)= V^{(l)}\x + a^{(l)}$ be the piecewise affine representation of the ReLU network $f_{\phi}$ on the finite number of affine regions $Q^{(l)}$ \citep{understanding-nn-relu}. Suppose that $V^{(l)}$ have independent rows, then for almost any $\x$ we have $|| f_{\phi}(\delta \cdot \x) || \underset{\delta \rightarrow \infty}{\rightarrow} \infty$. i.e the norm of the latent representations $\z_\delta=f_{\phi}(\delta \cdot \x)$ associated to the input $\delta \cdot \x$ goes to infinity.
\end{lemma}

\begin{proof}
We prove now lem.~\ref{lem:asymptotic-latent-norm}. Let $\x \in \real^\inputdim $ be a non-zero input and $f_{\phi}$ be a ReLU network. Lem.~\ref{lem:relu-regions} implies that there exists $\delta^* \in \real^{+}$ and $l \in \{1,..., R\}$ such that $\delta \cdot \x \in Q^{(l)}$ for all $\delta > \delta^*$. Thus, $\z_{\delta} = f_{\phi}(\delta \cdot \x) = \delta \cdot (V^{(l)} \x) + a^{(l)}$ for all $\delta > \delta^*$. Note that for $\delta\in [\delta^*, +\infty]$,  $\z_{\delta}$ follows an affine half line $S_{\x} = \{\z \condition \z = \delta \cdot (V^{(l)} \x) + a^{(l)}, \delta > \delta^* \}$ in the latent space. Further, note that $V^{(l)}\x \neq 0$ since $\x \neq 0$ and $V^{(l)}$ has independent rows. Therefore, we have $|| \z_\delta|| \underset{\delta \rightarrow \infty}{\rightarrow} + \infty$
\end{proof}

\begin{theorem}
\label{thm:dkl}
Let a Deep Kernel Learning model parametrized with a (deep) encoder $f_{\phi}$ with piecewise ReLU activations, a set of $K$ inducing points $\{\phi_{k}\}_{k=1}^{K}$ and a RBF, Matern or Rational Quadratic kernel $\kappa(\cdot, \cdot)$ \cite{expressing-structure-kernels, gp-for-ml}. Let $f_{\phi}(\x)= V^{(l)}\x + a^{(l)}$ be the piecewise affine representation of the ReLU network $f_{\phi}$ on the finite number of affine regions $Q^{(l)}$ \citep{understanding-nn-relu}. Suppose that $V^{(l)}$ have independent rows, then for almost any $\x$ we have \smash{$\sigma(f_{\phi}(\delta \cdot \x)) \underset{\delta \rightarrow \infty}{\rightarrow} c$ where $ c = \kappa(0, 0)$}.
\end{theorem}

\begin{proof}
We prove now thm.~\ref{thm:dkl}. Lem.~\ref{lem:relu-regions} says that $||\z_\delta|| \underset{\delta \rightarrow \infty}{\rightarrow} + \infty$ where $\z_\delta=f_{\phi}(\delta \cdot \x)$. It implies that $||\z_\delta - \bm{\phi}_k|| \underset{\delta \rightarrow \infty}{\rightarrow} \infty$ for all inducing point $\bm{\phi}_k$. Thus, we obtain $\kappa(\z_\delta, \bm{\phi}_k) \underset{\delta \rightarrow \infty}{\rightarrow} 0$ where $\kappa(\cdot, \cdot)$ is the RBF, Matern or Rational Quadratic kernel \cite{expressing-structure-kernels, gp-for-ml}. Since the variance of the predictive Gaussian distribution associated with the Gaussian process is $\sigma(f_{\phi}(\delta \cdot \x)) = c - \bm{\kappa} \bm{C} \bm{\kappa}$ where $c = \kappa(f_{\phi}(\delta \cdot \x), f_{\phi}(\delta \cdot \x)) = \kappa(0, 0)$, $\bm{\kappa}_k = \kappa(\z_\delta, \bm{\phi}_{k})$ and $\bm{\kappa}_{k, k'} = \kappa(\bm{\phi}_{k}, \bm{\phi}_{k'})$. This gives the final result $\sigma(f_{\phi}(\delta \cdot \x)) \underset{\delta \rightarrow \infty}{\rightarrow} c$ where $ c = \kappa(0, 0)$.
\end{proof}

Thm.~\ref{thm:dkl} implies that deep kernel learning on a latent space parametrized with a neural network is guaranteed to predict high uncertainty corresponding to the prior uncertainty far from training data. This includes the uncertainty predicted by the GP associated to each action $a$ in the combination of DQN and deep kernel learning presented in sec.~\ref{sec:models}. The uncertainty prediction $u_\text{epist}(s_t, a_t) = \entropy(\DNormal(\mu(\s^{(t)}, a^{(t)}), \sigma(\s^{(t)}, a^{(t)})))$ becomes high for input states $s\datatx$ extremely different from the training environment i.e. $||s\datatx|| \rightarrow \infty$.

\begin{theorem}
\label{thm:natpn}
\cite{natpn} Let a Natural Posterior Network model parametrized with a (deep) encoder $f_{\phi}$ with piecewise ReLU activations, a decoder $g_{\psi}$ and the density $\prob(\z \condition \bm{\omega})$. Let $f_{\phi}(\x)= V^{(l)}\x + a^{(l)}$ be the piecewise affine representation of the ReLU network $f_{\phi}$ on the finite number of affine regions $Q^{(l)}$ \citep{understanding-nn-relu}. Suppose that $V^{(l)}$ have independent rows and the density function $\prob(\z \condition \bm{\omega})$ has bounded derivatives, then for almost any $\x$ we have \smash{$\prob(f_{\phi}(\delta \cdot \x) \condition \bm{\omega}) \underset{\delta \rightarrow \infty}{\rightarrow} 0$}. i.e the evidence becomes small far from training data.
\end{theorem}

The proof of thm.~\ref{thm:natpn} is given in \cite{natpn}, and relies also on lem.~\ref{lem:asymptotic-latent-norm} and the fact that a smooth density estimator should converge to $0$ far from training data. Intuitively, it implies that the epistemic associated to each possible action $a$ by the combination of DQN and posterior network becomes high for input states $s\datatx$ extremely different from the training environment i.e. $||s\datatx|| \rightarrow \infty$. In particular, prior parameter takes over in the posterior update (i.e. $\evidence^\text{post}(\s^{(t)}, a)) \rightarrow \evidence^\text{prior}$, $\priorparam^\text{post}(\s^{(t)}, a) \rightarrow \priorparam^\text{prior}$)

\section{Model Details}
\label{app:models-details}

\looseness=-1
We train all models on a single GPU (NVIDIA GTX 1080 Ti or NVIDIA GTX 2080 Ti, 11 GB memory). All models use the same core architecture. They use a $2$ layers MLP with 128 hidden units for the CartPole environment, a $2$ layers MLP with 64 hidden units for the Acrobot environment and a $3$ layers MLP with 128 hidden units for the LunarLander environment. All models are trained using $5$ random seeds with the Adam optimizer \cite{adam-optimizer}. For fair comparison, we use the same hyperparameters for the DQN architecture in all uncertainty models: the target network parameters are completely updated (i.e. $\tau=1.$) every $10$ training iterations. The epsilon-greedy strategy start with $\epsilon=1.$ and decay till $\epsilon=0.01$ after $1000$ iteration steps. The discount factor is set to $0.99$. Further, we use a batch size of $16$, a replay size of $1000$ and a maximum number of training iterations of $13000$ for Cartpole, a batch size of $64$, a replay size of $10000$ and a maximum number of training iterations of $120000$ for Acrobot, and a batch size of $128$, a replay size of $10000$ and a maximum number of training iterations of $300000$ for LunarLander. For each type of uncertainty model, we performed a grid search for the learning rate in the range $[10^{-1}, 10^{-4}]$. 

\looseness=-1
Each uncertainty method has also its own hyperparameters. We show an hyperparameter study in app.~\ref{app:hyper-parameter-study} for the main hyper-parameters of each uncertainty method. For the MC dropout model, we make a grid-search over the number of samples $n \in [10, 20, 40, 80]$ and the drop probability $p \in [.1, .2, .3, .4, .5]$. In the main experiments, we use $n=80$ and $p=.2$. For the ensemble model, we make a grid-search over the number of networks $n \in [10, 20, 40, 80]$. In the main experiments, we use $n=80$. For the deep kernel learning model, we make a grid-search over the number of inducing points $n \in [10, 20, 40, 80]$, the latent dimension $H \in [16, 32, 64]$, the kernel type in RBF, RQ and Matern-$\frac{3}{2}$ Kernel, and an ELBO regularization factor in $\lambda \in [., 1.]$ . In the main experiments, we use $n=80$ inducing points, a latent dimension of $H=64$, the RQ kernel, and a regularization factor of $\lambda=.1$. Further, we observed that adding a batch normalization layer right after the encoder $f_{\bm{\theta}}$ was stabilizing the training similarly to \citet{postnet}. For the evidential network model based on posterior networks, we make a grid-search over the flow depth $d \in [8, 16, 32]$, the latent dimension $H \in [8, 16, 32]$. In the main experiments, we use a radial flow with depth $d=8$ and a latent dimension of $H=16$. Further, we observed that adding a batch normalization layer right after the encoder $f_{\bm{\theta}}$ was stabilizing the training similarly to \citet{postnet}.

\looseness=-1
We will provide the github repository with the code on the project page 
\url{https://www.cs.cit.tum.de/daml/aleatoric-epistemic-uncertainty-rl/}.
%\footnote{\label{link:code}\url{https://anonymous.4open.science/r/Aleatoric-Epistemic-Uncertainty-RL-DDB5/README.md}}. 
To conduct the experiments, we used Pytorch \cite{pytorch} with BSD license, Pytorch Lightning \citep{pytorch-lightning} with Apache 2.0 license and Weight\&Biases \cite{wandb}. Further, we also use GPytorch for to implement the deep kernel model \cite{gpytorch}.

\section{Environment Details}
\label{app:environments-details}

\looseness=-1
We use OpenAI gym environments \cite{gym} with MIT license. We design the OOD environments such that they should not be relevant to the original training environment task, and thus being a reasonable failure mode. Further, we design a continuum of perturbed environments going from tasks very similar to the training environment to the tasks very different from the original environment. We distinguish between perturbations on the \emph{state} space, the \emph{action} space, and the \emph{transition} dynamics to follow the MDP structure of the original environment. In contrast, \cite{assessing-generalization-rl, benchmark-ood-detection-rl} mostly focus on perturbations on the environment parameters. We will provide the github repository with the code for the OOD and perturbed environment on the project page 
\url{https://www.cs.cit.tum.de/daml/aleatoric-epistemic-uncertainty-rl/}.
%\footnote{\label{link:code}\url{https://anonymous.4open.science/r/Aleatoric-Epistemic-Uncertainty-RL-DDB5/README.md}}. 

\looseness=-1
\textbf{Cartpole \citep{cartpole}} In this environment, the goal of the agent is to maintain a pole on a cart straight up. This environment has a discrete action space with $2$ possible actions corresponding to apply the a force to the left or the right of the cart. This environment has a continuous state space with dimension $4$ corresponds to. The episode ends when the pole is more than 15 degrees from vertical, or the cart moves more than 2.4 units from the center. The reward is $+1$ at every time step that the pole stays up. The maximum length of an episode is 200 steps. For the OOD environment, the input states are drawn from a Gaussian distribution with unit variance i.e. $\s^{(\tdata), \text{pert}} \sim \DNormal(0, 1)$. For the perturbed environments with perturbation strength $\epsilon$, the action space is perturbed by randomly adding Gaussian noise to the scale of the force applied to the cart (i.e. $f^{\text{pert}} = (1. + x)f$ where $f \in \{-10, 10\}$ is default action force and $x \sim \DNormal(0, \epsilon)$ is the perturbation), the state space is perturbed by adding Gaussian noise to the observation scale (i.e. $\s^{(\tdata), \text{pert}} = (1. + x)\s\datatx$ where $x \sim \DNormal(0, \epsilon)$ is the perturbation), and the transition dynamic is perturbed by adding a uniform noise centered around the true dynamic parameters (i.e. $\nu^{\text{pert}} = (1 + x) \nu$ where $x \sim U(-\epsilon, \epsilon)$ where $\nu$ is the original environment parameters) such as the gravity, pole length ...etc.

\looseness=-1
\textbf{Acrobot \citep{acrobot1, acrobot2}} In this environment, the agent control a robot arm with two links and its goal is to move the end of the lower link up to a given height. This environment has a discrete action space with $3$ possible actions corresponding to apply a positive torque, a negative torque or nothing. This environment has a continuous state space with dimension $6$ corresponding to the $4$ joint angles and the $2$ angular velocities. The episode ends when the lower link of the robot arm is above a given height. The reward is $-1$ at every time step that the pole does not reach the expected height. The maximum length of an episode is 500 steps at most. For the OOD environment, the input states are drawn from a Gaussian distribution with unit variance (i.e. $\s^{(\tdata), \text{pert}} \sim \DNormal(0, 1)$). For the perturbed environments with perturbation strength $\epsilon$, the action space is perturbed by randomly sampling actions with probability $p=\frac{\epsilon}{2}$, the state space is perturbed by adding Gaussian noise to the observation scale (i.e. $\s^{(\tdata), \text{pert}} = (1. + x)\s\datatx$ where $x \sim \DNormal(0, \epsilon)$), and the transition dynamic is perturbed by adding a uniform noise centered around the true dynamic parameters (i.e. $\nu^{\text{pert}} = (1 + x) \nu$ where $x \sim U(-\epsilon, \epsilon)$ where $\nu$ is the original environment parameters) such as the lengths of the links, the masses of the links.

\looseness=-1
\textbf{LunarLander \citep{lunarlander1, lunarlander2}} In this environment, the agent control a space ship and its goal is to land it on the surface of the moon. This environment has a discrete action space with $4$ possible actions corresponding to apply a torque to the left, to the right, downward or nothing. This environment has a continuous state space with dimension $8$ corresponding to the space ship coordinates. The reward is correlated with fast landing in the correct area without crashes. The episode ends when the spaceship is landed or crashed. For the OOD environment, the input states are drawn from a Gaussian distribution with unit variance (i.e. $\s^{(\tdata), \text{pert}} \sim \DNormal(0, 1)$). For the perturbed environments with perturbation strength $\epsilon$, the action space is perturbed by randomly sampling actions with probability $p=\frac{\epsilon}{2}$, the state space is perturbed by adding Gaussian noise to the observation scale (i.e. $\s^{(\tdata), \text{pert}} = (1. + x)\s\datatx$ where $x \sim \DNormal(0, \epsilon)$), and the transition dynamic is perturbed by adding a uniform noise centered around the true dynamic parameters (i.e. $\nu^{\text{pert}} = (1 + x) \nu$ where $x \sim U(-\epsilon, \epsilon)$ where $\nu$ is the original environment parameters) such as the lengths of the links, the masses of the links.

\section{Metric Details - Training Time}
\label{app:training-time-metric-details}

\looseness=-1
%\textbf{Training time.} 
We track the current reward, the epistemic uncertainty and the aleatoric uncertainty at every training step. The epistemic and aleatoric uncertainty are defined by the variance or the entropy of the epistemic and the aleatoric distributions (see sec.~\ref{sec:models}).  The two uncertainty types are then normalized between $[0, 1]$ with min-max normalization to compute the relative epistemic and aleatoric uncertainty on the plots. The normalization enable an easier comparison of the trend of the uncertainty estimates across methods. For all these experiments, we compute the mean and the standard error of the mean across $5$ seeds for all results.

\section{Metric Details - Testing Time}
\label{app:testing-time-metric-details}

\looseness=-1
%\textbf{Testing time.} 
We save $20$ model checkpoints at regular interval during the whole training. We evaluate then the $20$ checkpointed models at testing time. First, we compute the in-distribution (ID) reward average over $10$ episodes on the original training environment. Second, we compute the OOD detection scores by comparing the epistemic uncertainty of the ID and the OOD environment over $10$ episodes each with the area under the receiver operating characteristic curve (AUC-ROC) and the area under the precision-recall curve (AUC-PR). Higher scores indicate better OOD detection performances. Third, we compute the averaged reward and epistemic uncertainty on the perturbed environment over $10$ episodes. For all these experiments, we compute the mean and the standard error of the mean across $5$ seeds for all results. Further, we also sampled $5$ random perturbations for each perturbation strength.

\section{Additional Experiments}
\label{app:additional-experiments}

\subsection{Training Time}

We show additional results on CartPole, Acrobot and LunarLander in fig.~\ref{fig:model-training-performance-cartpole}, fig.~\ref{fig:model-training-performance-acrobot} and fig.~\ref{fig:model-training-performance-lunarlander} to compare the performance of the uncertainty estimates of the four uncertainty methods at training time. The epistemic uncertainty estimates of PostNet decrease during training. Thus, PostNet empirically validate des.~\ref{ax:training_state}. Further, Ensemble and PostNet require a low number of finished episodes on CartPole and LuncarLander. This translates for these two envionments into a safer learning with a lower number of restart of the systems.

\begin{figure}
    \centering
        \begin{subfigure}{.5\textwidth}
        \includegraphics[width=\textwidth]{resources/legend.png}
    \end{subfigure}
    \vspace{-5mm}
    
    \begin{subfigure}{.245\textwidth}
        \includegraphics[width=\textwidth]{resources/cartpole-training_total_reward-training-model.png}
    \end{subfigure}
    \begin{subfigure}{.245\textwidth}
        \includegraphics[width=\textwidth]{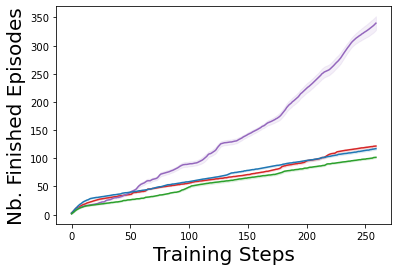}  
    \end{subfigure}
    \begin{subfigure}{.245\textwidth}
        \includegraphics[width=\textwidth]{resources/cartpole-training_epistemic_uncertainty-training-model.png}
    \end{subfigure}
    \begin{subfigure}{.245\textwidth}
        \includegraphics[width=\textwidth]{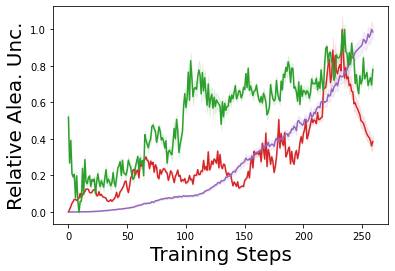}  
    \end{subfigure}
    \caption{Comparison of the training performance of the four uncertainty methods using epsilon-greedy strategies on CartPole. Ideally, a uncertainty aware-model should achieve high reward with few samples and episodes and with a decreasing epistemic uncertainty.}
    \label{fig:model-training-performance-cartpole}
\end{figure}
\begin{figure}
    \centering
        \begin{subfigure}{.5\textwidth}
        \includegraphics[width=\textwidth]{resources/legend.png}
    \end{subfigure}
    \vspace{-5mm}
    
    \begin{subfigure}{.245\textwidth}
        \includegraphics[width=\textwidth]{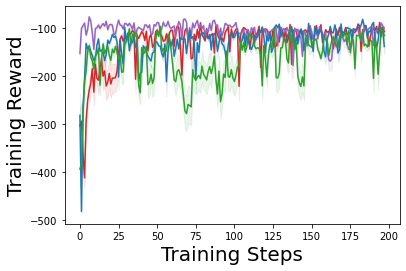}
    \end{subfigure}
    \begin{subfigure}{.245\textwidth}
        \includegraphics[width=\textwidth]{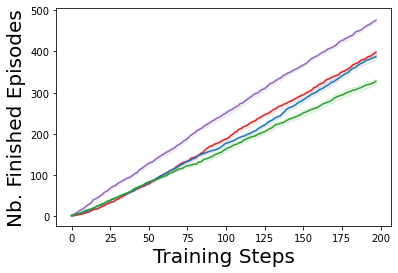}  
    \end{subfigure}
    \begin{subfigure}{.245\textwidth}
        \includegraphics[width=\textwidth]{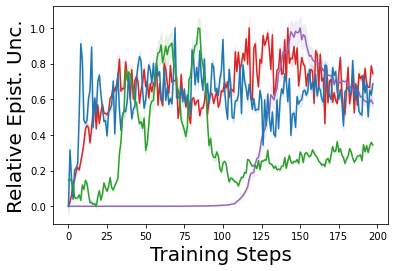}
    \end{subfigure}
    \begin{subfigure}{.245\textwidth}
        \includegraphics[width=\textwidth]{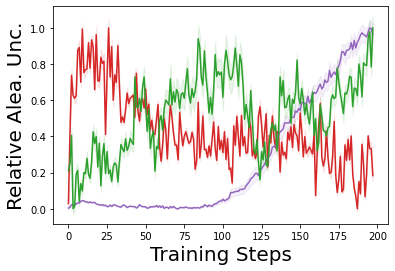}  
    \end{subfigure}
    \caption{Comparison of the training performance of the four uncertainty methods using epsilon-greedy strategies on Acrobot. Ideally, a uncertainty aware-model should achieve high reward with few samples and with a decreasing epistemic uncertainty.}
    \label{fig:model-training-performance-acrobot}
\end{figure}
\begin{figure}
    \centering
        \begin{subfigure}{.5\textwidth}
        \includegraphics[width=\textwidth]{resources/legend.png}
    \end{subfigure}
    \vspace{-5mm}
    
    \begin{subfigure}{.245\textwidth}
        \includegraphics[width=\textwidth]{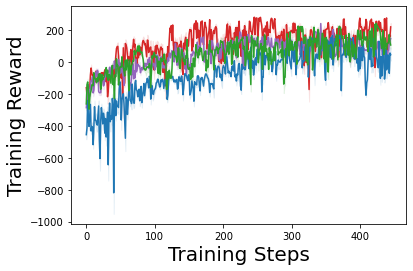}
    \end{subfigure}
    \begin{subfigure}{.245\textwidth}
        \includegraphics[width=\textwidth]{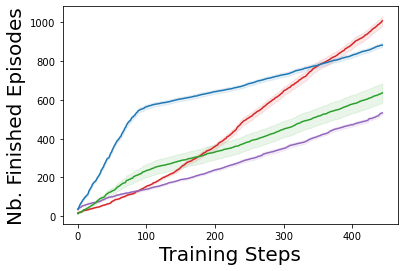}  
    \end{subfigure}
    \begin{subfigure}{.245\textwidth}
        \includegraphics[width=\textwidth]{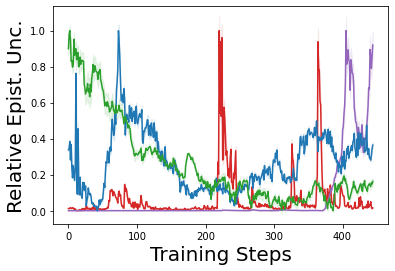}
    \end{subfigure}
    \begin{subfigure}{.245\textwidth}
        \includegraphics[width=\textwidth]{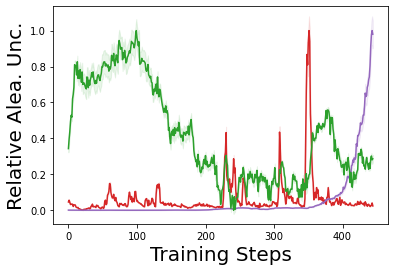}  
    \end{subfigure}
    \caption{Comparison of the training performance of the four uncertainty methods using epsilon-greedy strategies on LunarLander. Ideally, a uncertainty aware-model should achieve high reward  with few samples and episodes and with a decreasing epistemic uncertainty.}
    \label{fig:model-training-performance-lunarlander}
\end{figure}

We show additional results in fig.~\ref{fig:strategy-training-performance-cartpole}, fig.~\ref{fig:strategy-training-performance-acrobot} and fig.~\ref{fig:strategy-training-performance-lunarlander} to compare the performance of the sampling-epistemic and the sampling-aleatoric strategies at training time. The sampling-epistemic strategy consistently achieve a better sample efficiency. Thus, Ensemble, DropOut and PostNet empirically satisfy des.~\ref{ax:training_strategy}. Hence, disentangling aleatoric and epistemic uncertainty can speed learning in a training environment.

\begin{figure}
    \centering
    \begin{subfigure}{.45\textwidth}
        \includegraphics[width=\textwidth]{resources/sampling-legend.png}
    \end{subfigure}
    \vspace{-3mm}
    
    \begin{subfigure}{.245\textwidth}
        \includegraphics[width=\textwidth]{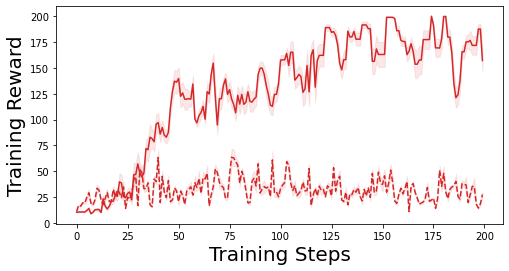}
    \end{subfigure}
    \begin{subfigure}{.245\textwidth}
        \includegraphics[width=\textwidth]{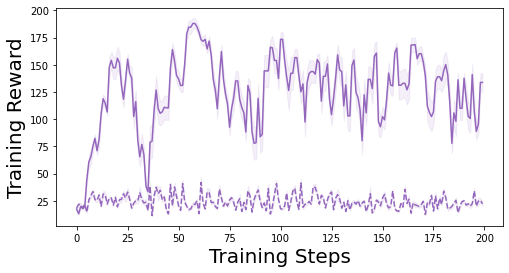}
    \end{subfigure}
    \begin{subfigure}{.245\textwidth}
        \includegraphics[width=\textwidth]{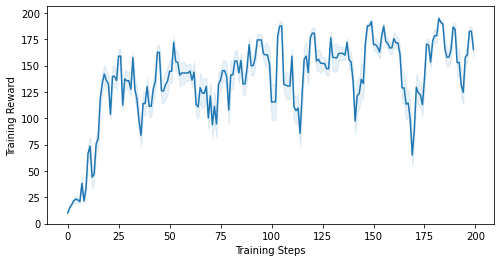}
    \end{subfigure}
    \begin{subfigure}{.245\textwidth}
        \includegraphics[width=\textwidth]{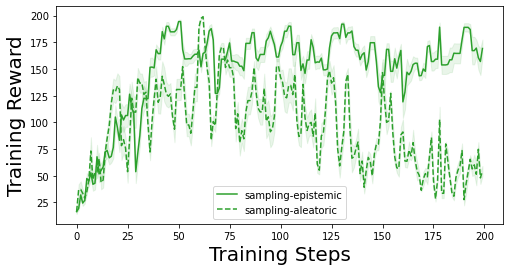}
    \end{subfigure}
    
    \begin{subfigure}{.245\textwidth}
        \includegraphics[width=\textwidth]{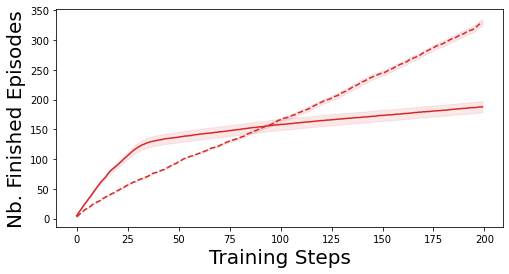}
    \end{subfigure}
    \begin{subfigure}{.245\textwidth}
        \includegraphics[width=\textwidth]{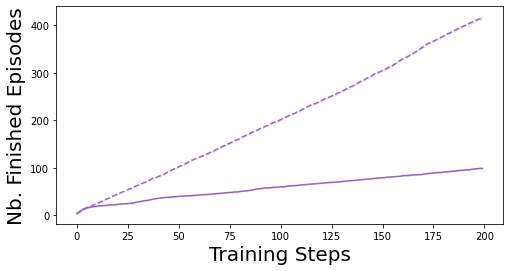}
    \end{subfigure}
    \begin{subfigure}{.245\textwidth}
        \includegraphics[width=\textwidth]{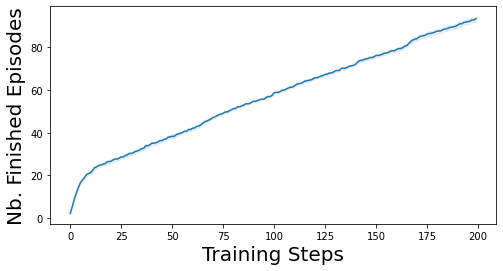}
    \end{subfigure}
    \begin{subfigure}{.245\textwidth}
        \includegraphics[width=\textwidth]{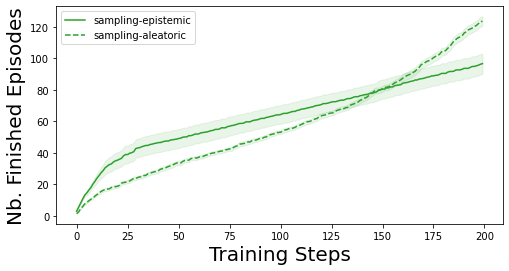}
    \end{subfigure}
    \caption{Comparison of the training performance on Cartpole. The four uncertainty methods use the sampling-aleatoric or the sampling-epistemic at training time. Ideally, an uncertainty aware-model should high reward with few samples.}
    \label{fig:strategy-training-performance-cartpole}
\end{figure}
\begin{figure}
    \centering
    \begin{subfigure}{.45\textwidth}
        \includegraphics[width=\textwidth]{resources/sampling-legend.png}
    \end{subfigure}
    \vspace{-3mm}
    
    \begin{subfigure}{.245\textwidth}
        \includegraphics[width=\textwidth]{resources/acrobot-training_total_reward-dropout-training-strategy.png}
    \end{subfigure}
    \begin{subfigure}{.245\textwidth}
        \includegraphics[width=\textwidth]{resources/acrobot-training_total_reward-ensemble-training-strategy.png}
    \end{subfigure}
    \begin{subfigure}{.245\textwidth}
        \includegraphics[width=\textwidth]{resources/acrobot-training_total_reward-dkl-training-strategy.png}
    \end{subfigure}
    \begin{subfigure}{.245\textwidth}
        \includegraphics[width=\textwidth]{resources/acrobot-training_total_reward-postnet-training-strategy.png}
    \end{subfigure}
    
    \begin{subfigure}{.245\textwidth}
        \includegraphics[width=\textwidth]{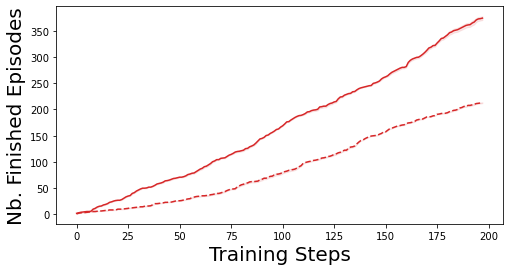}
    \end{subfigure}
    \begin{subfigure}{.245\textwidth}
        \includegraphics[width=\textwidth]{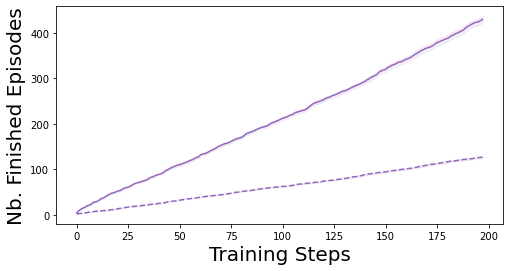}
    \end{subfigure}
    \begin{subfigure}{.245\textwidth}
        \includegraphics[width=\textwidth]{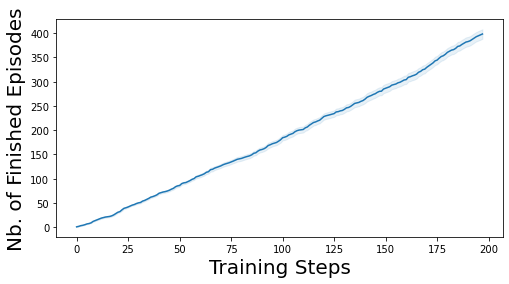}
    \end{subfigure}
    \begin{subfigure}{.245\textwidth}
        \includegraphics[width=\textwidth]{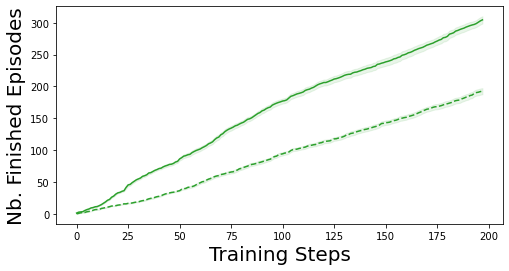}
    \end{subfigure}
    \caption{Comparison of the training performance on Acrobot. The four uncertainty methods use the sampling-aleatoric or the sampling-epistemic at training time. Ideally, an uncertainty aware-model should high reward with few samples.}
    \label{fig:strategy-training-performance-acrobot}
\end{figure}
\begin{figure}
    \centering
    \begin{subfigure}{.45\textwidth}
        \includegraphics[width=\textwidth]{resources/sampling-legend.png}
    \end{subfigure}
    \vspace{-3mm}
    
    \begin{subfigure}{.245\textwidth}
        \includegraphics[width=\textwidth]{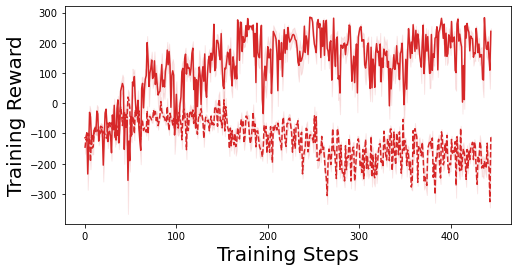}
    \end{subfigure}
    \begin{subfigure}{.245\textwidth}
        \includegraphics[width=\textwidth]{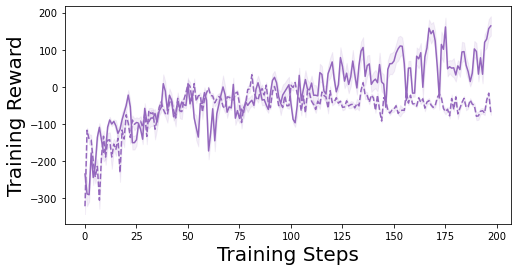}
    \end{subfigure}
    \begin{subfigure}{.245\textwidth}
        \includegraphics[width=\textwidth]{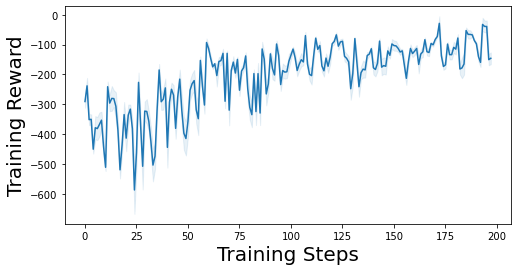}
    \end{subfigure}
    \begin{subfigure}{.245\textwidth}
        \includegraphics[width=\textwidth]{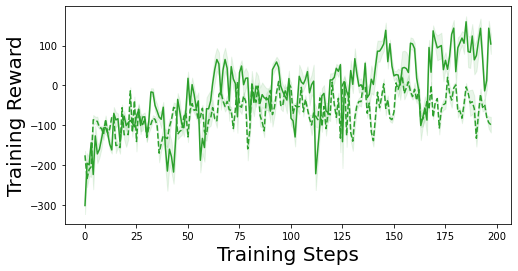}
    \end{subfigure}
    
    \begin{subfigure}{.245\textwidth}
        \includegraphics[width=\textwidth]{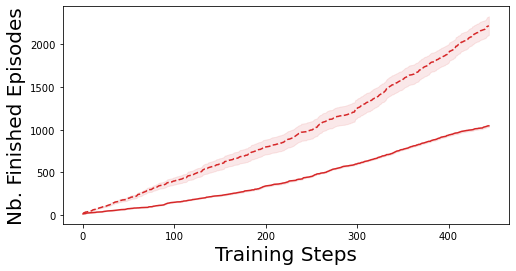}
    \end{subfigure}
    \begin{subfigure}{.245\textwidth}
        \includegraphics[width=\textwidth]{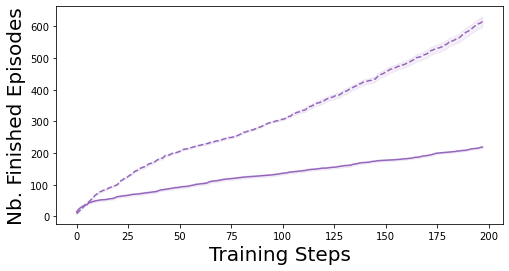}
    \end{subfigure}
    \begin{subfigure}{.245\textwidth}
        \includegraphics[width=\textwidth]{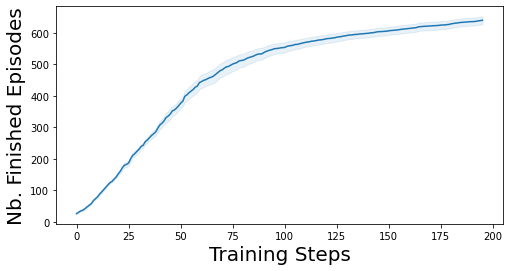}
    \end{subfigure}
    \begin{subfigure}{.245\textwidth}
        \includegraphics[width=\textwidth]{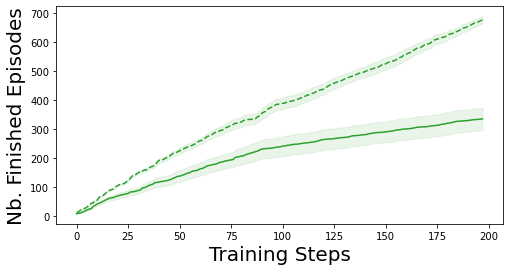}
    \end{subfigure}
    \caption{Comparison of the training performance on LunarLander. The four uncertainty methods use the sampling-aleatoric or the sampling-epistemic at training time. Ideally, an uncertainty aware-model should high reward with few samples.}
    \label{fig:strategy-training-performance-lunarlander}
\end{figure}

\subsection{Testing Time}

We show additional results in fig.~\ref{fig:model-testing-performance-cartpole}, fig.~\ref{fig:model-testing-performance-acrobot} and fig.~\ref{fig:model-testing-performance-lunarlander} to compare the generalization and OOD detection performance of the uncertainty estimates of the four uncertainty methods at testing time. The models use the sampling-epistemic or the sampling-aleatoric strategy at both training and testing time. Further, we show other additional results for OOD detection by using the area under the precision-recall (AUC-PR) scores instead of the area under the receiver operating characteristic curve (AUC-ROC) in fig.~\ref{fig:strategy-testing-ood-auc-pr-performance-cartpole}, fig.~\ref{fig:strategy-testing-ood-auc-pr-performance-acrobot}, fig.~\ref{fig:strategy-testing-ood-auc-pr-performance-lunarlander}. We observe that DKL and PostNet achieve  very high OOD detection scores in most settings  compared to DropOut and Ensemble. These \emph{empirical} results align with the \emph{theoretical} results stating that DKL and PostNet should assign high uncertainty to states very different from states observed during training. Thus, DKL and PostNet validate des.~\ref{ax:testing_state}. In particular, DKL and PostNet can reliably equip DQN with epistemic uncertainty estimates which can be used to flag anomalous OOD states.

\begin{figure}
    \centering
    \vspace{-7mm}
        \begin{subfigure}{.5\textwidth}
        \includegraphics[width=\textwidth]{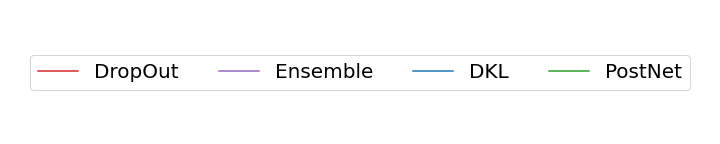}
    \end{subfigure}
    \vspace{-5mm}
    
    \begin{subfigure}{.4\textwidth}
        \includegraphics[width=\textwidth]{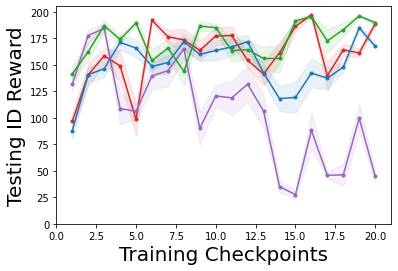}  
    \end{subfigure}
    \begin{subfigure}{.4\textwidth}
        \includegraphics[width=\textwidth]{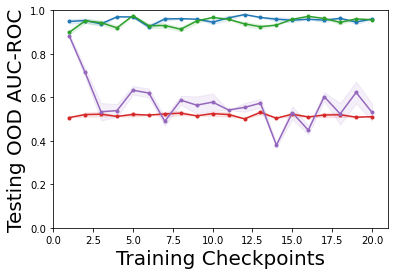}
    \end{subfigure}
        \vspace{-3mm}
    \caption{Comparison of the testing performance of the four uncertainty methods using epsilon-greedy strategies at training and testing time on CartPole. Ideally, an uncertainty aware-model should achieve high reward and high OOD detection scores.}
    \label{fig:model-testing-performance-cartpole}
    \vspace{-4mm}
\end{figure}
\begin{figure}
    \centering
        \begin{subfigure}{.5\textwidth}
        \includegraphics[width=\textwidth]{resources/legend.png}
    \end{subfigure}
    \vspace{-5mm}
    
    \begin{subfigure}{.4\textwidth}
        \includegraphics[width=\textwidth]{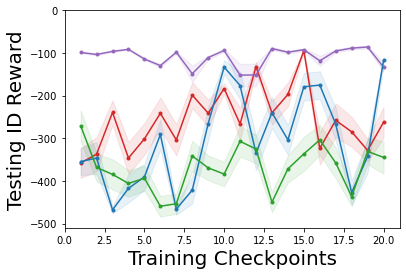}  
    \end{subfigure}
    \begin{subfigure}{.4\textwidth}
        \includegraphics[width=\textwidth]{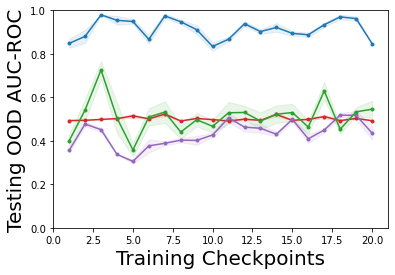}
    \end{subfigure}
    \caption{Comparison of the testing performance of the four uncertainty methods using epsilon-greedy strategies at training and testing time on Acrobot. Ideally, an uncertainty aware-model should achieve high reward and high OOD detection scores.}
    \label{fig:model-testing-performance-acrobot}
\end{figure}
\begin{figure}
    \centering
        \begin{subfigure}{.5\textwidth}
        \includegraphics[width=\textwidth]{resources/legend.png}
    \end{subfigure}
    \vspace{-5mm}
    
    \begin{subfigure}{.4\textwidth}
        \includegraphics[width=\textwidth]{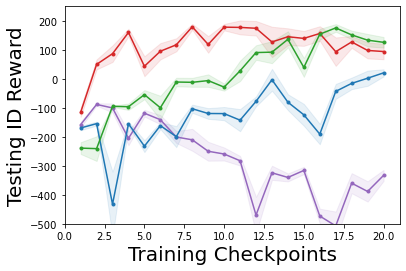}  
    \end{subfigure}
    \begin{subfigure}{.4\textwidth}
        \includegraphics[width=\textwidth]{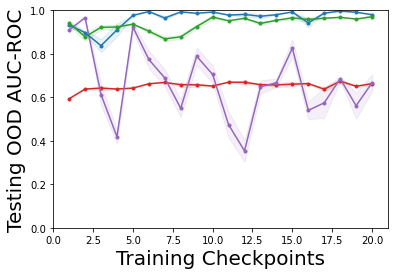}
    \end{subfigure}
    \caption{Comparison of the testing performance of the four uncertainty methods using epsilon-greedy strategies at training and testing time on LunarLander. Ideally, an uncertainty aware-model should achieve high reward and high OOD detection scores.}
    \label{fig:model-testing-performance-lunarlander}
\end{figure}

We show additional results in fig.~\ref{fig:strategy-testing-performance-cartpole}, fig.~\ref{fig:strategy-testing-performance-acrobot} and fig.~\ref{fig:strategy-testing-performance-lunarlander} to compare the performance of the sampling-epistemic and sampling-aleatoric strategies for each uncertainty model. All models use the same epsilon-greedy strategy at training time. We observe that the sampling-epistemic strategy is consistently better than sampling-aleatoric at testing time. The higher generalization capacity of the sampling-epistemic strategy aligns with \cite{epistemic-pomdp} which recasts the problem of generalization in RL as solving an epistemic POMDP. These empirical results underline the need to disentangle both aleatoric and epistemic uncertainty for high reward performance at testing time.

\begin{figure}
    \centering
    \begin{subfigure}{.45\textwidth}
        \includegraphics[width=\textwidth]{resources/sampling-legend.png}
    \end{subfigure}
    \vspace{-3mm}
    
    \begin{subfigure}{.245\textwidth}
        \includegraphics[width=\textwidth]{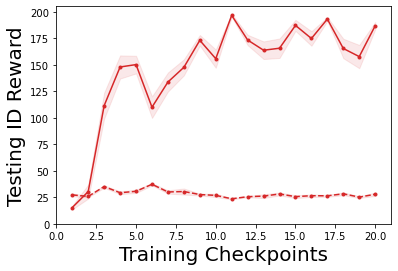}
    \end{subfigure}
    \begin{subfigure}{.245\textwidth}
        \includegraphics[width=\textwidth]{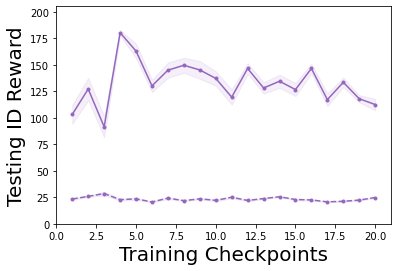}
    \end{subfigure}
    \begin{subfigure}{.245\textwidth}
        \includegraphics[width=\textwidth]{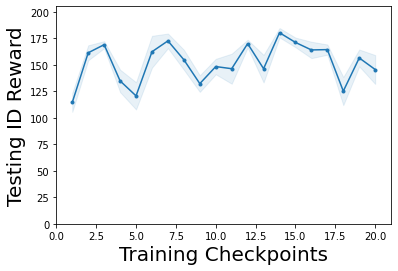}
    \end{subfigure}
    \begin{subfigure}{.245\textwidth}
        \includegraphics[width=\textwidth]{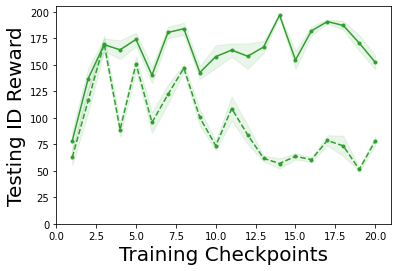}
    \end{subfigure}
    
    \begin{subfigure}{.245\textwidth}
        \includegraphics[width=\textwidth]{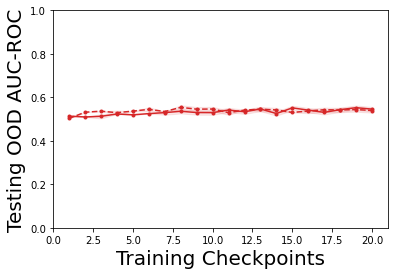}
    \end{subfigure}
    \begin{subfigure}{.245\textwidth}
        \includegraphics[width=\textwidth]{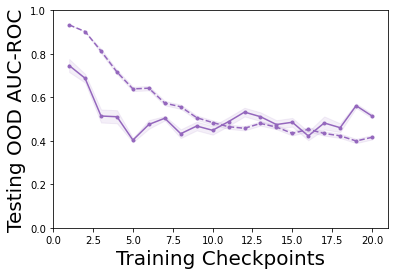}
    \end{subfigure}
    \begin{subfigure}{.245\textwidth}
        \includegraphics[width=\textwidth]{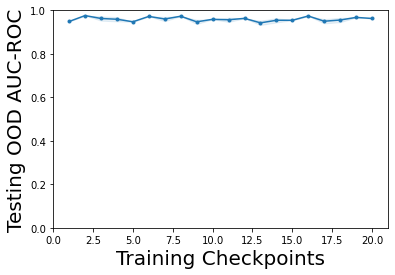}
    \end{subfigure}
    \begin{subfigure}{.245\textwidth}
        \includegraphics[width=\textwidth]{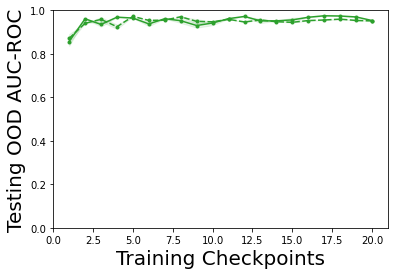}
    \end{subfigure}
    \caption{Comparison of the testing reward and OOD performance on CartPole. The four uncertainty methods use the sampling-aleatoric or sampling-epistemic strategies at both training and testing time. Ideally, an uncertainty aware-model should achieve high testing reward and high OOD AUC-ROC detection score.}
    \label{fig:strategy-testing-performance-cartpole}
\end{figure}
\begin{figure}
    \centering
    \begin{subfigure}{.45\textwidth}
        \includegraphics[width=\textwidth]{resources/sampling-legend.png}
    \end{subfigure}
    \vspace{-3mm}
    
    \begin{subfigure}{.245\textwidth}
        \includegraphics[width=\textwidth]{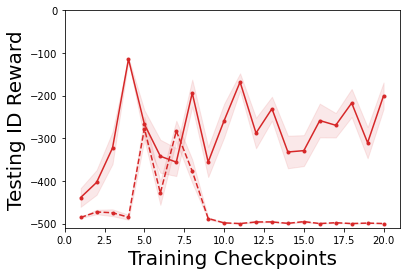}
    \end{subfigure}
    \begin{subfigure}{.245\textwidth}
        \includegraphics[width=\textwidth]{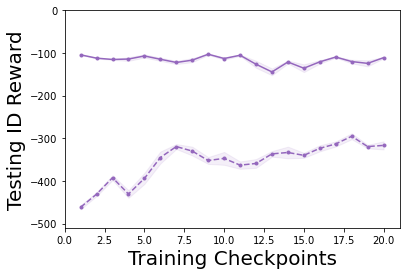}
    \end{subfigure}
    \begin{subfigure}{.245\textwidth}
        \includegraphics[width=\textwidth]{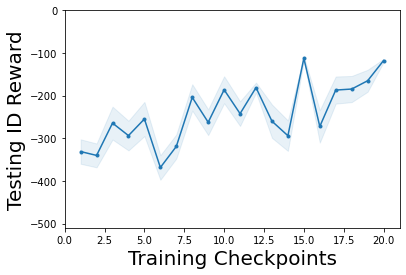}
    \end{subfigure}
    \begin{subfigure}{.245\textwidth}
        \includegraphics[width=\textwidth]{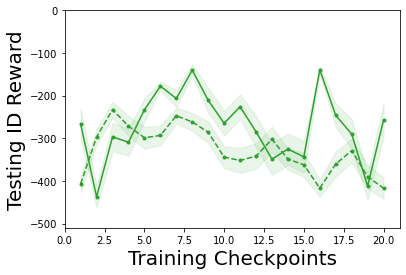}
    \end{subfigure}
    
    \begin{subfigure}{.245\textwidth}
        \includegraphics[width=\textwidth]{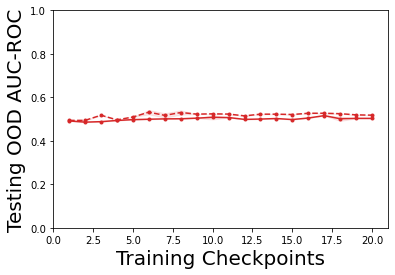}
    \end{subfigure}
    \begin{subfigure}{.245\textwidth}
        \includegraphics[width=\textwidth]{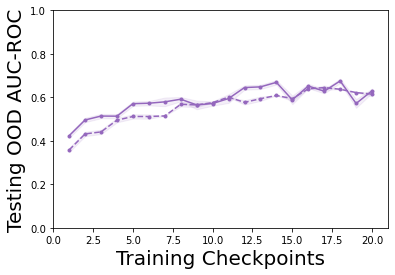}
    \end{subfigure}
    \begin{subfigure}{.245\textwidth}
        \includegraphics[width=\textwidth]{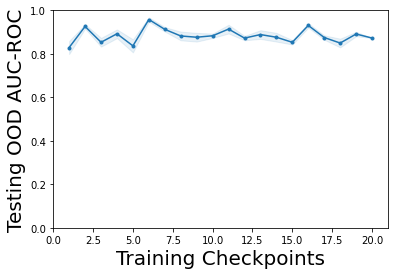}
    \end{subfigure}
    \begin{subfigure}{.245\textwidth}
        \includegraphics[width=\textwidth]{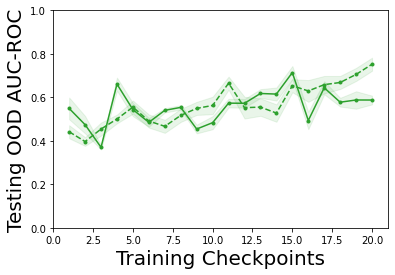}
    \end{subfigure}
    \caption{Comparison of the testing reward and OOD performance on Acrobot. The four uncertainty methods use the sampling-aleatoric or sampling-epistemic strategies at both training and testing time. Ideally, an uncertainty aware-model should achieve high testing reward and high OOD AUC-ROC detection score.}
    \label{fig:strategy-testing-performance-acrobot}
\end{figure}

We show additional results in fig.~\ref{fig:strategy-state-shift-testing-performance-cartpole}, fig.~\ref{fig:strategy-state-shift-testing-performance-acrobot}, fig.~\ref{fig:strategy-state-shift-testing-performance-lunarlander}, fig.~\ref{fig:strategy-action-shift-testing-performance-cartpole}, fig.~\ref{fig:strategy-action-shift-testing-performance-acrobot}, fig.~\ref{fig:strategy-action-shift-testing-performance-lunarlander}, fig.~\ref{fig:strategy-transition-shift-testing-performance-cartpole}, fig.~\ref{fig:strategy-transition-shift-testing-performance-acrobot}, fig.~\ref{fig:strategy-transition-shift-testing-performance-lunarlander} to compare the generalization and uncertainty performances of the sampling-epistemic and sampling-aleatoric strategies of each method on perturbed environments with state, action and transition dynamic perturbations. All methods achieve lower reward on environment with stronger perturbations. This is expected since a model cannot generalize to all new environments. The sampling-epistemic strategy achieves significantly better that the sampling-aleatoric strategy. The generalization capacity of the sampling-epistemic strategy aligns again with \cite{epistemic-pomdp}. Thus, differentiating between aleatoric and epistemic uncertainty can improve generalization. Finally, only DKL and PostNet reliably assign higher epistemic uncertainty to most of the perturbation types. Therefore, DKL and PostNet have a good trade-off between generalization and detection of new perturbed environments.

\textbf{Video:} For a better visualization, we provide supplementary videos on the project page \url{https://www.cs.cit.tum.de/daml/aleatoric-epistemic-uncertainty-rl/}. The videos show the landing performance, the reward performance, and the relative epistemic uncertainty prediction of the PostNet model in the original LunarLander environments and two environments with perturbed states with perturbation strengths equal to $0.5$ and $2.0$. On the original environment, we observe that the space ship lands correctly with lower epistemic uncertainty after landing. On the perturbed environment with strength $0.5$, we observe that the space ship avoids crashing but assigns higher epistemic uncertainty when moving further from the landing zone. Finally, on the perturbed environment with strength $2.0$, we observe that the space ship assigns significantly higher epistemic uncertainty especially when approaching the floor before the crash.

\begin{figure}
    \centering
    \begin{subfigure}{.45\textwidth}
        \includegraphics[width=\textwidth]{resources/sampling-legend.png}
    \end{subfigure}
    \vspace{-3mm}
    
    \begin{subfigure}{.245\textwidth}
        \includegraphics[width=\textwidth]{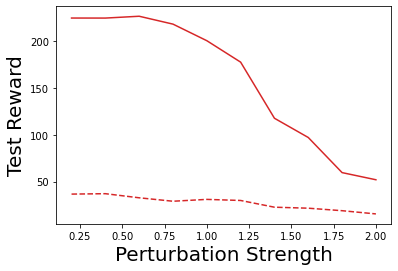}
    \end{subfigure}
    \begin{subfigure}{.245\textwidth}
        \includegraphics[width=\textwidth]{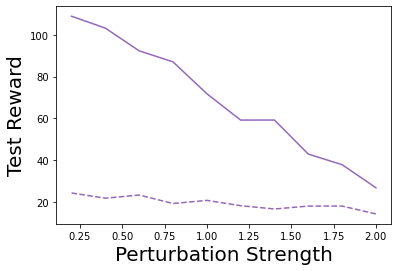}
    \end{subfigure}
    \begin{subfigure}{.245\textwidth}
        \includegraphics[width=\textwidth]{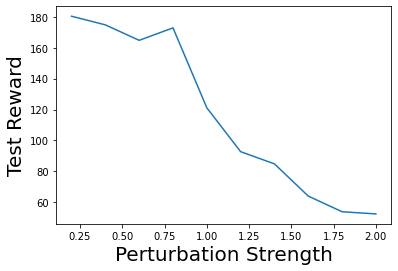}
    \end{subfigure}
    \begin{subfigure}{.245\textwidth}
        \includegraphics[width=\textwidth]{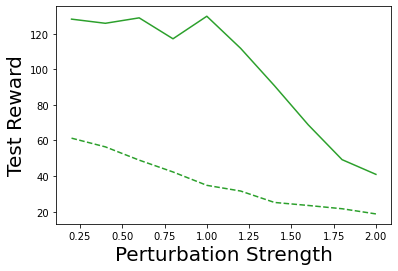}
    \end{subfigure}
    
    \begin{subfigure}{.245\textwidth}
        \includegraphics[width=\textwidth]{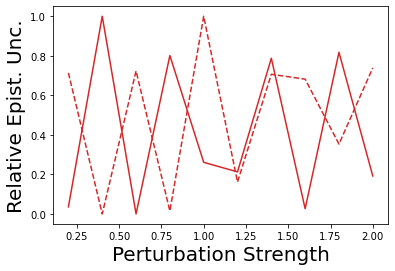}
    \end{subfigure}
    \begin{subfigure}{.245\textwidth}
        \includegraphics[width=\textwidth]{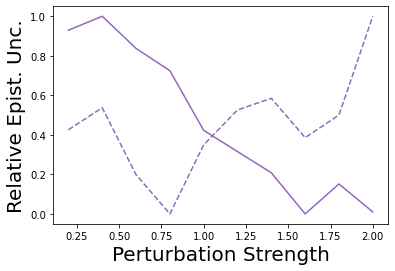}
    \end{subfigure}
    \begin{subfigure}{.245\textwidth}
        \includegraphics[width=\textwidth]{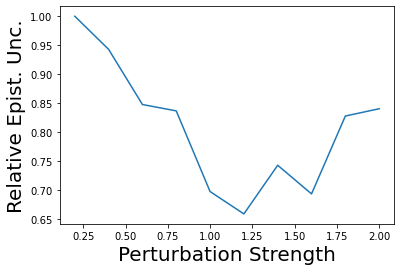}
    \end{subfigure}
    \begin{subfigure}{.245\textwidth}
        \includegraphics[width=\textwidth]{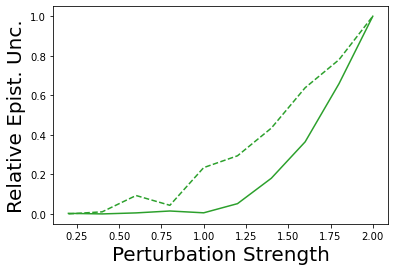}
    \end{subfigure}
    \caption{Comparison of the testing performance and the epistemic uncertainty predictions on CartPole with perturbed actions. The four uncertainty methods use the epsilon-greedy strategy at training time and the sampling-aleatoric or sampling-epistemic strategy at testing time. Ideally, an uncertainty-aware model should maintain high reward while assigning higher epistemic uncertainty on more severe perturbations.}
    \label{fig:strategy-action-shift-testing-performance-cartpole}
\end{figure}
\begin{figure}
    \centering
        \vspace{-3mm}
    \begin{subfigure}{.45\textwidth}
        \includegraphics[width=\textwidth]{resources/sampling-legend.png}
    \end{subfigure}
    \vspace{-3mm}
    
    \begin{subfigure}{.245\textwidth}
        \includegraphics[width=\textwidth]{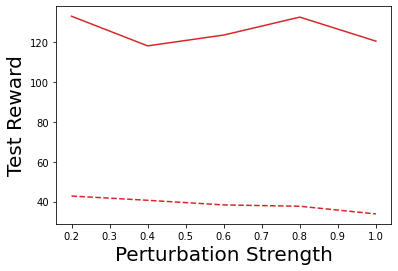}
    \end{subfigure}
    \begin{subfigure}{.245\textwidth}
        \includegraphics[width=\textwidth]{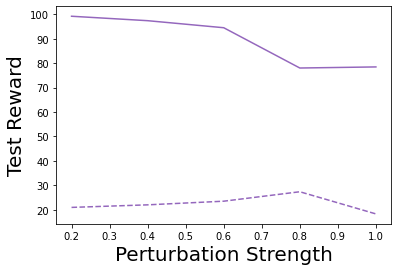}
    \end{subfigure}
    \begin{subfigure}{.245\textwidth}
        \includegraphics[width=\textwidth]{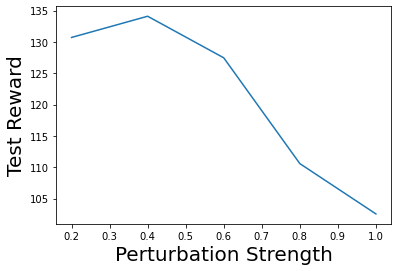}
    \end{subfigure}
    \begin{subfigure}{.245\textwidth}
        \includegraphics[width=\textwidth]{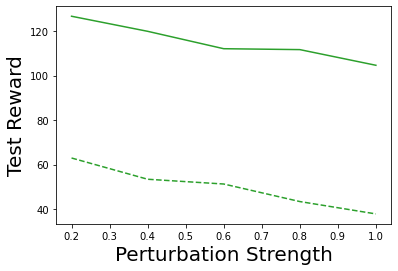}
    \end{subfigure}
    
    \begin{subfigure}{.245\textwidth}
        \includegraphics[width=\textwidth]{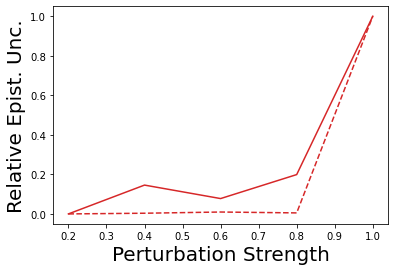}
    \end{subfigure}
    \begin{subfigure}{.245\textwidth}
        \includegraphics[width=\textwidth]{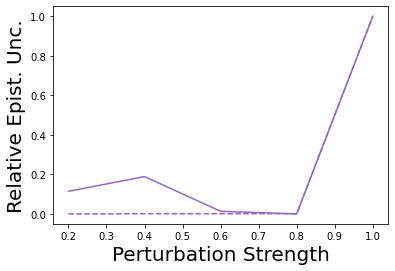}
    \end{subfigure}
    \begin{subfigure}{.245\textwidth}
        \includegraphics[width=\textwidth]{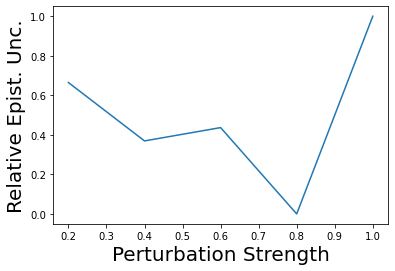}
    \end{subfigure}
    \begin{subfigure}{.245\textwidth}
        \includegraphics[width=\textwidth]{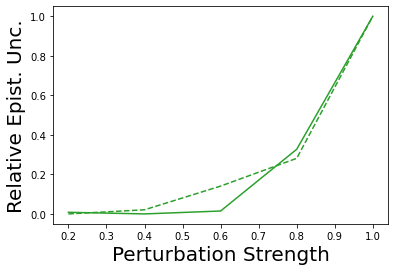}
    \end{subfigure}
        \vspace{-2mm}
    \caption{Comparison of the testing performance and the epistemic uncertainty predictions on CartPole with perturbed transitions. The four uncertainty methods use the epsilon-greedy strategy at training time and the sampling-aleatoric or sampling-epistemic strategy at testing time. Ideally, an uncertainty-aware model should maintain high reward while assigning higher epistemic uncertainty on more severe perturbations.}
    \label{fig:strategy-transition-shift-testing-performance-cartpole}
        \vspace{-3mm}
\end{figure}
\begin{figure}
    \centering
        \vspace{-6mm}
    \begin{subfigure}{.45\textwidth}
        \includegraphics[width=\textwidth]{resources/sampling-legend.png}
    \end{subfigure}
    \vspace{-3mm}
    
    \begin{subfigure}{.245\textwidth}
        \includegraphics[width=\textwidth]{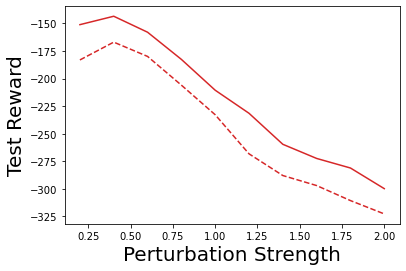}
    \end{subfigure}
    \begin{subfigure}{.245\textwidth}
        \includegraphics[width=\textwidth]{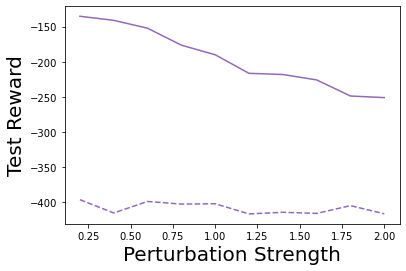}
    \end{subfigure}
    \begin{subfigure}{.245\textwidth}
        \includegraphics[width=\textwidth]{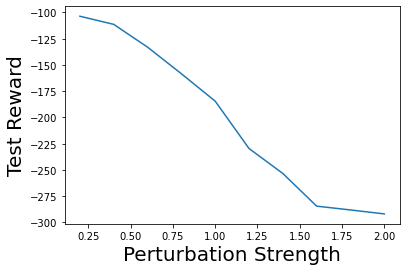}
    \end{subfigure}
    \begin{subfigure}{.245\textwidth}
        \includegraphics[width=\textwidth]{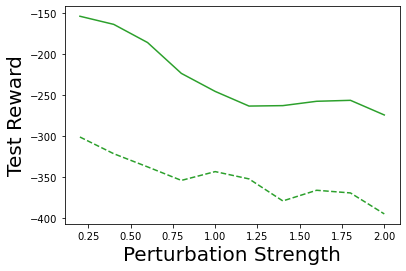}
    \end{subfigure}
    
    \begin{subfigure}{.245\textwidth}
        \includegraphics[width=\textwidth]{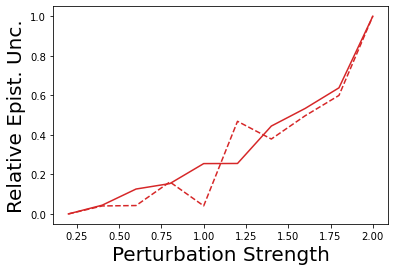}
    \end{subfigure}
    \begin{subfigure}{.245\textwidth}
        \includegraphics[width=\textwidth]{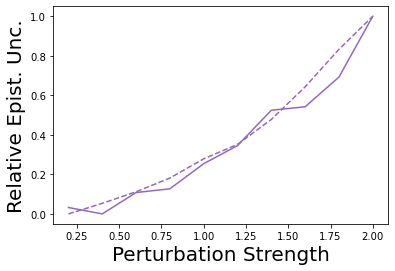}
    \end{subfigure}
    \begin{subfigure}{.245\textwidth}
        \includegraphics[width=\textwidth]{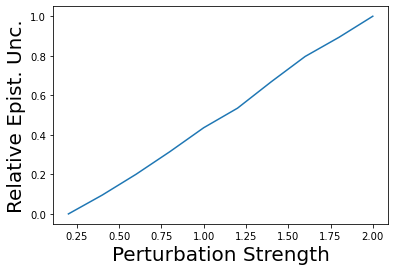}
    \end{subfigure}
    \begin{subfigure}{.245\textwidth}
        \includegraphics[width=\textwidth]{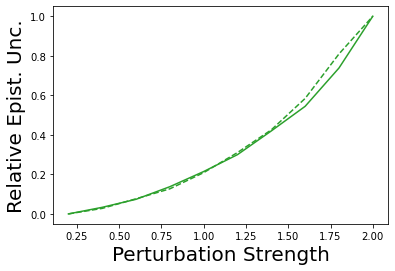}
    \end{subfigure}
        \vspace{-2mm}
    \caption{Comparison of the testing performance and the epistemic uncertainty predictions on Acrobot with perturbed states. The four uncertainty methods use the epsilon-greedy strategy at training time and the sampling-aleatoric or sampling-epistemic strategy at testing time. Ideally, an uncertainty-aware model should maintain high reward while assigning higher epistemic uncertainty on more severe perturbations.}
    \label{fig:strategy-state-shift-testing-performance-acrobot}
        \vspace{-6mm}
\end{figure}
\begin{figure}
    \centering
    \begin{subfigure}{.45\textwidth}
        \includegraphics[width=\textwidth]{resources/sampling-legend.png}
    \end{subfigure}
    \vspace{-3mm}
    
    \begin{subfigure}{.245\textwidth}
        \includegraphics[width=\textwidth]{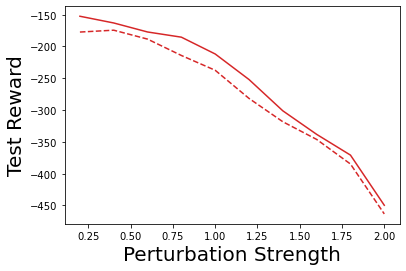}
    \end{subfigure}
    \begin{subfigure}{.245\textwidth}
        \includegraphics[width=\textwidth]{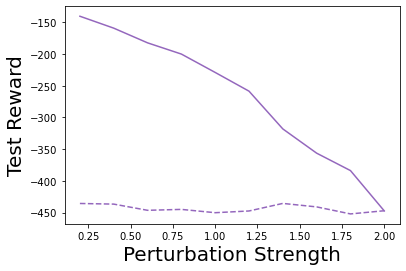}
    \end{subfigure}
    \begin{subfigure}{.245\textwidth}
        \includegraphics[width=\textwidth]{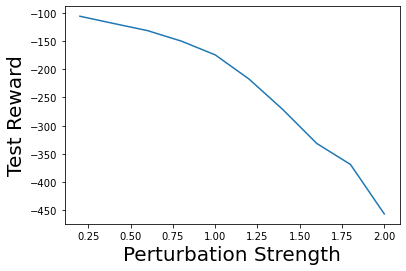}
    \end{subfigure}
    \begin{subfigure}{.245\textwidth}
        \includegraphics[width=\textwidth]{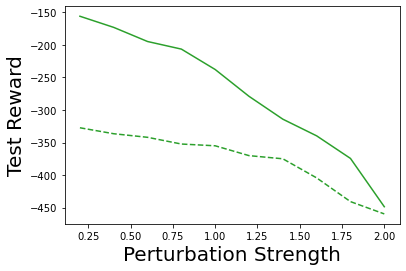}
    \end{subfigure}
    
    \begin{subfigure}{.245\textwidth}
        \includegraphics[width=\textwidth]{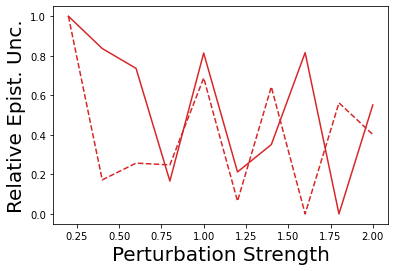}
    \end{subfigure}
    \begin{subfigure}{.245\textwidth}
        \includegraphics[width=\textwidth]{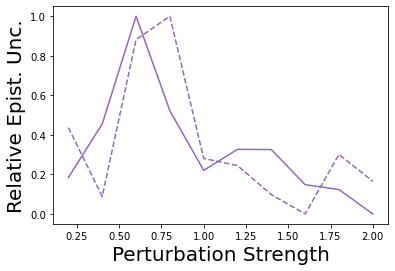}
    \end{subfigure}
    \begin{subfigure}{.245\textwidth}
        \includegraphics[width=\textwidth]{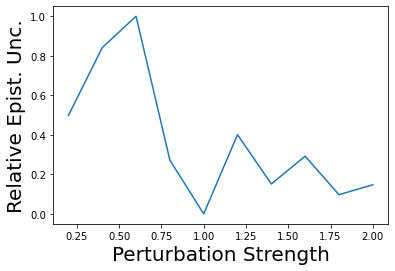}
    \end{subfigure}
    \begin{subfigure}{.245\textwidth}
        \includegraphics[width=\textwidth]{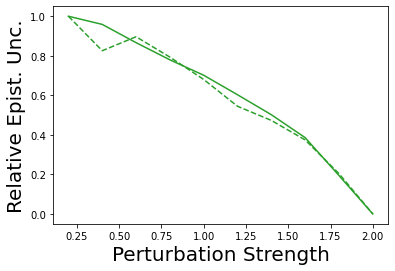}
    \end{subfigure}
    \caption{Comparison of the testing performance and the epistemic uncertainty predictions on Acrobot with perturbed actions. The four uncertainty methods use the epsilon-greedy strategy at training time and the sampling-aleatoric or sampling-epistemic strategy at testing time. Ideally, an uncertainty-aware model should maintain high reward while assigning higher epistemic uncertainty on more severe perturbations.}
    \label{fig:strategy-action-shift-testing-performance-acrobot}
\end{figure}
\begin{figure}
    \centering
        \vspace{-3mm}
    \begin{subfigure}{.45\textwidth}
        \includegraphics[width=\textwidth]{resources/sampling-legend.png}
    \end{subfigure}
    \vspace{-3mm}
    
    \begin{subfigure}{.245\textwidth}
        \includegraphics[width=\textwidth]{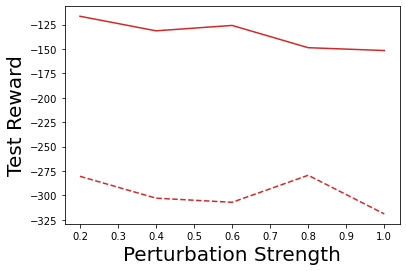}
    \end{subfigure}
    \begin{subfigure}{.245\textwidth}
        \includegraphics[width=\textwidth]{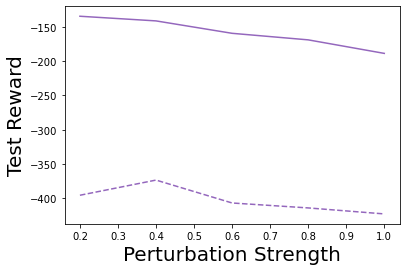}
    \end{subfigure}
    \begin{subfigure}{.245\textwidth}
        \includegraphics[width=\textwidth]{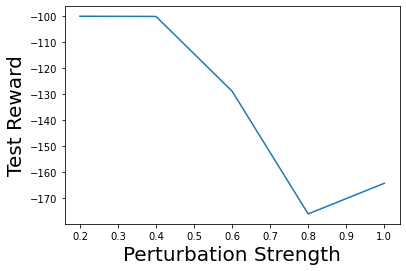}
    \end{subfigure}
    \begin{subfigure}{.245\textwidth}
        \includegraphics[width=\textwidth]{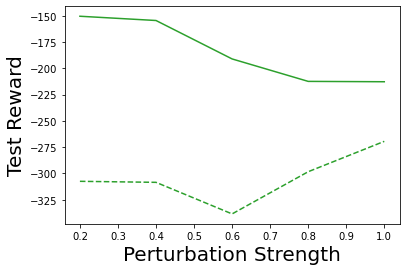}
    \end{subfigure}
    
    \begin{subfigure}{.245\textwidth}
        \includegraphics[width=\textwidth]{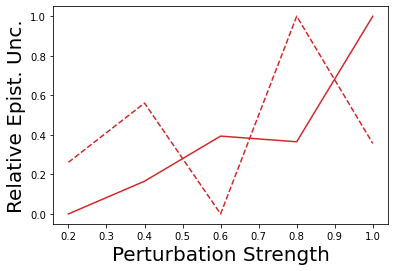}
    \end{subfigure}
    \begin{subfigure}{.245\textwidth}
        \includegraphics[width=\textwidth]{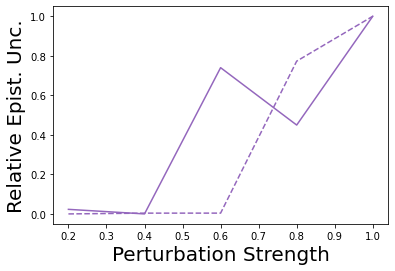}
    \end{subfigure}
    \begin{subfigure}{.245\textwidth}
        \includegraphics[width=\textwidth]{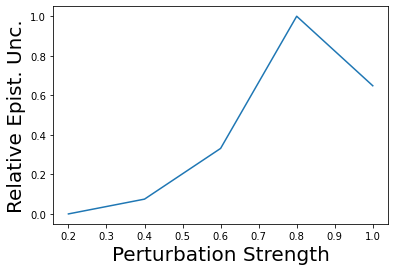}
    \end{subfigure}
    \begin{subfigure}{.245\textwidth}
        \includegraphics[width=\textwidth]{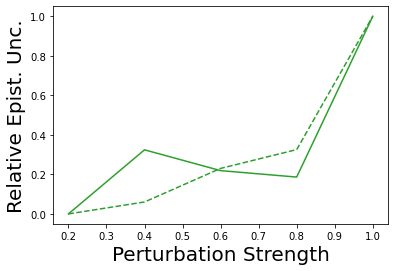}
    \end{subfigure}
        \vspace{-2mm}
    \caption{Comparison of the testing performance and the epistemic uncertainty predictions on Acrobot with perturbed transitions. The four uncertainty methods use the epsilon-greedy strategy at training time and the sampling-aleatoric or sampling-epistemic strategy at testing time. Ideally, an uncertainty-aware model should maintain high reward while assigning higher epistemic uncertainty on more severe perturbations.}
    \label{fig:strategy-transition-shift-testing-performance-acrobot}
        \vspace{-3mm}
\end{figure}
\begin{figure}
    \centering
        \vspace{-6mm}
    \begin{subfigure}{.45\textwidth}
        \includegraphics[width=\textwidth]{resources/sampling-legend.png}
    \end{subfigure}
    \vspace{-3mm}
    
    \begin{subfigure}{.245\textwidth}
        \includegraphics[width=\textwidth]{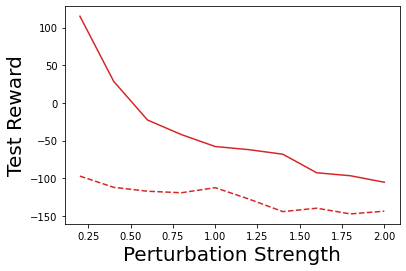}
    \end{subfigure}
    \begin{subfigure}{.245\textwidth}
        \includegraphics[width=\textwidth]{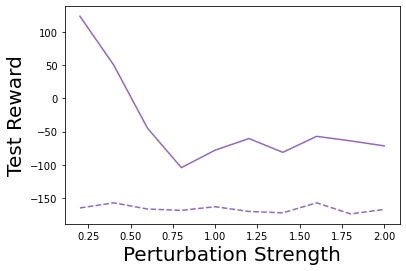}
    \end{subfigure}
    \begin{subfigure}{.245\textwidth}
        \includegraphics[width=\textwidth]{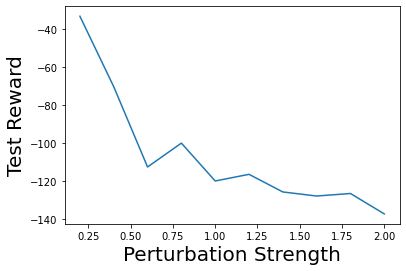}
    \end{subfigure}
    \begin{subfigure}{.245\textwidth}
        \includegraphics[width=\textwidth]{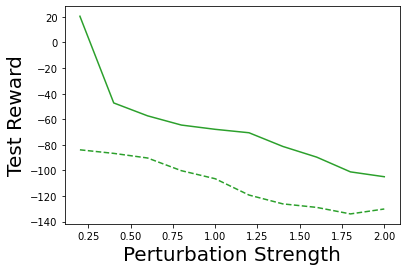}
    \end{subfigure}
    
    \begin{subfigure}{.245\textwidth}
        \includegraphics[width=\textwidth]{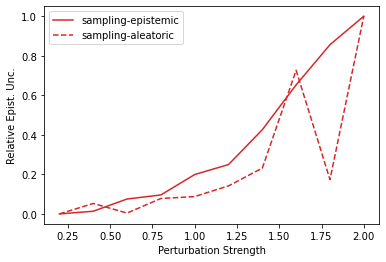}
    \end{subfigure}
    \begin{subfigure}{.245\textwidth}
        \includegraphics[width=\textwidth]{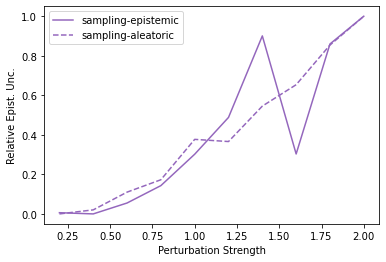}
    \end{subfigure}
    \begin{subfigure}{.245\textwidth}
        \includegraphics[width=\textwidth]{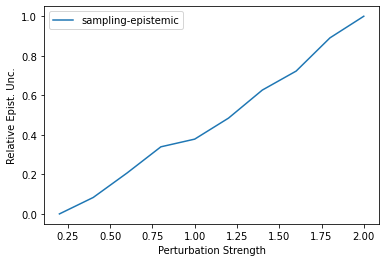}
    \end{subfigure}
    \begin{subfigure}{.245\textwidth}
        \includegraphics[width=\textwidth]{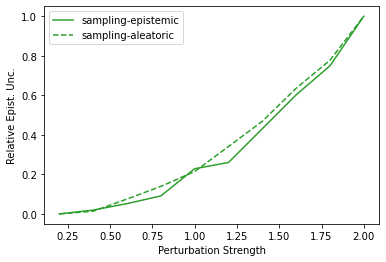}
    \end{subfigure}
        \vspace{-2mm}
    \caption{Comparison of the testing performance and the epistemic uncertainty predictions on LunarLander with perturbed states. The four uncertainty methods use the epsilon-greedy strategy at training time and the sampling-aleatoric or sampling-epistemic strategy at testing time. Ideally, an uncertainty-aware model should maintain high reward while assigning higher epistemic uncertainty on more severe perturbations.}
    \label{fig:strategy-state-shift-testing-performance-lunarlander}
        \vspace{-6mm}
\end{figure}
\begin{figure}
    \centering
    \begin{subfigure}{.45\textwidth}
        \includegraphics[width=\textwidth]{resources/sampling-legend.png}
    \end{subfigure}
    \vspace{-3mm}
    
    \begin{subfigure}{.245\textwidth}
        \includegraphics[width=\textwidth]{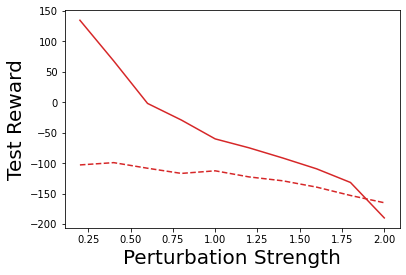}
    \end{subfigure}
    \begin{subfigure}{.245\textwidth}
        \includegraphics[width=\textwidth]{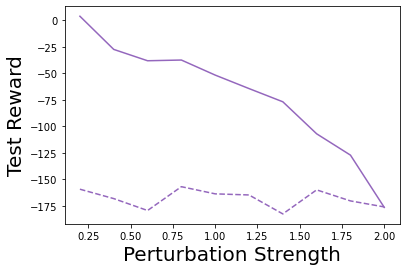}
    \end{subfigure}
    \begin{subfigure}{.245\textwidth}
        \includegraphics[width=\textwidth]{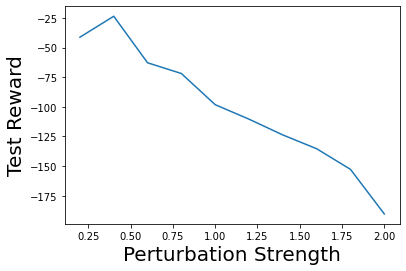}
    \end{subfigure}
    \begin{subfigure}{.245\textwidth}
        \includegraphics[width=\textwidth]{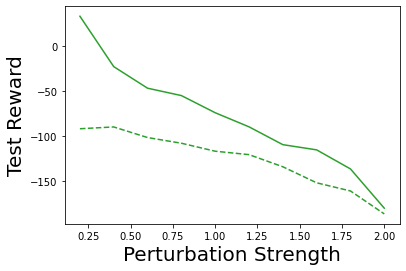}
    \end{subfigure}
    
    \begin{subfigure}{.245\textwidth}
        \includegraphics[width=\textwidth]{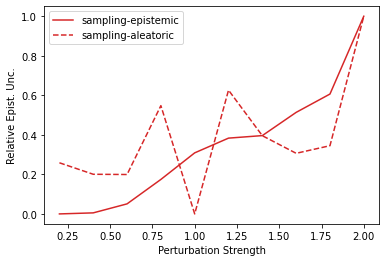}
    \end{subfigure}
    \begin{subfigure}{.245\textwidth}
        \includegraphics[width=\textwidth]{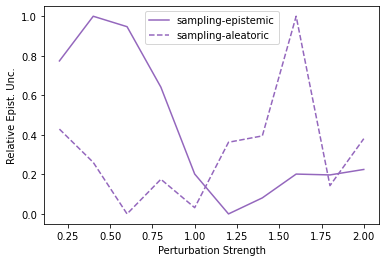}
    \end{subfigure}
    \begin{subfigure}{.245\textwidth}
        \includegraphics[width=\textwidth]{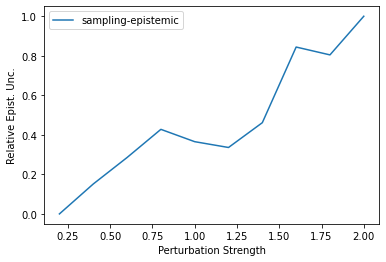}
    \end{subfigure}
    \begin{subfigure}{.245\textwidth}
        \includegraphics[width=\textwidth]{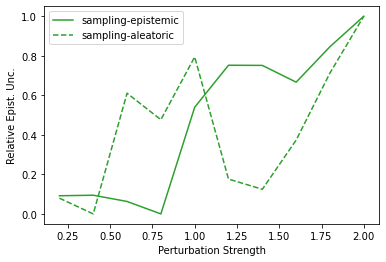}
    \end{subfigure}
    \caption{Comparison of the testing performance and the epistemic uncertainty predictions on LunarLander with perturbed actions. The four uncertainty methods use the epsilon-greedy strategy at training time and the sampling-aleatoric or sampling-epistemic strategy at testing time. Ideally, an uncertainty-aware model should maintain high reward while assigning higher epistemic uncertainty on more severe perturbations.}
    \label{fig:strategy-action-shift-testing-performance-lunarlander}
\end{figure}
\begin{figure}
    \centering
        \vspace{-3mm}
    \begin{subfigure}{.45\textwidth}
        \includegraphics[width=\textwidth]{resources/sampling-legend.png}
    \end{subfigure}
    \vspace{-3mm}
    
    \begin{subfigure}{.245\textwidth}
        \includegraphics[width=\textwidth]{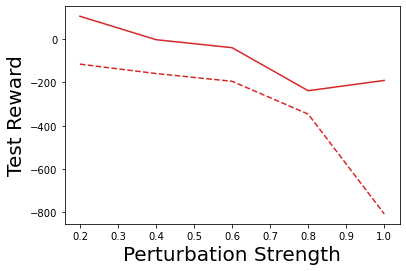}
    \end{subfigure}
    \begin{subfigure}{.245\textwidth}
        \includegraphics[width=\textwidth]{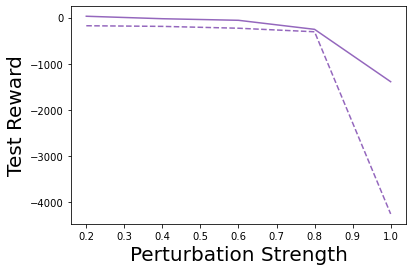}
    \end{subfigure}
    \begin{subfigure}{.245\textwidth}
        \includegraphics[width=\textwidth]{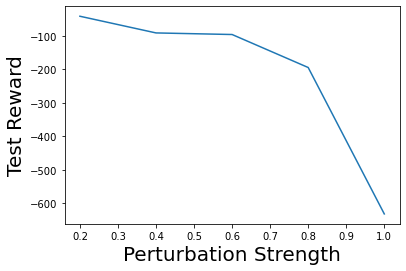}
    \end{subfigure}
    \begin{subfigure}{.245\textwidth}
        \includegraphics[width=\textwidth]{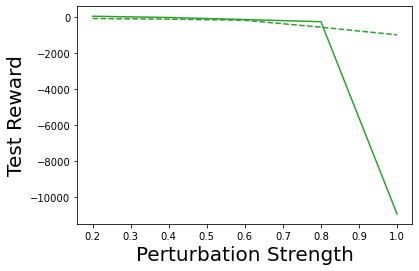}
    \end{subfigure}
    
    \begin{subfigure}{.245\textwidth}
        \includegraphics[width=\textwidth]{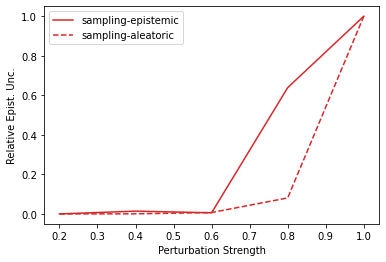}
    \end{subfigure}
    \begin{subfigure}{.245\textwidth}
        \includegraphics[width=\textwidth]{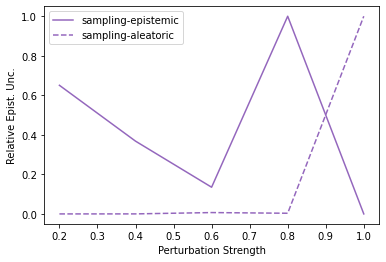}
    \end{subfigure}
    \begin{subfigure}{.245\textwidth}
        \includegraphics[width=\textwidth]{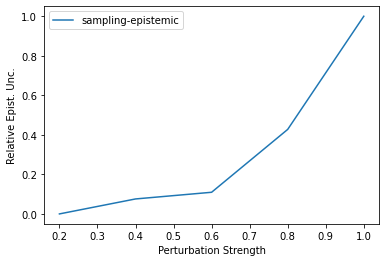}
    \end{subfigure}
    \begin{subfigure}{.245\textwidth}
        \includegraphics[width=\textwidth]{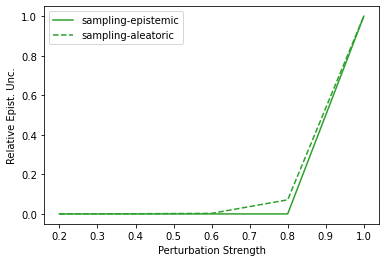}
    \end{subfigure}
        \vspace{-2mm}
    \caption{Comparison of the testing performance and the epistemic uncertainty predictions on LunarLander with perturbed transitions. The four uncertainty methods use the epsilon-greedy strategy at training time and the sampling-aleatoric or sampling-epistemic strategy at testing time. Ideally, an uncertainty-aware model should maintain high reward while assigning higher epistemic uncertainty on more severe perturbations.}
    \label{fig:strategy-transition-shift-testing-performance-lunarlander}
        \vspace{-3mm}
\end{figure}

\begin{figure}
    \centering
        \vspace{-3mm}
    \begin{subfigure}{.45\textwidth}
        \includegraphics[width=\textwidth]{resources/sampling-legend.png}
    \end{subfigure}
    \vspace{-3mm}
    
    \begin{subfigure}{.245\textwidth}
        \includegraphics[width=\textwidth]{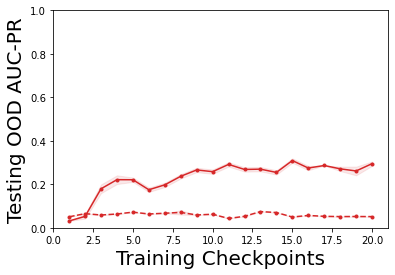}
    \end{subfigure}
    \begin{subfigure}{.245\textwidth}
        \includegraphics[width=\textwidth]{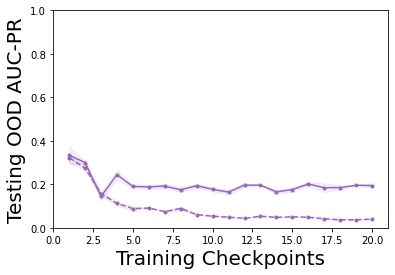}
    \end{subfigure}
    \begin{subfigure}{.245\textwidth}
        \includegraphics[width=\textwidth]{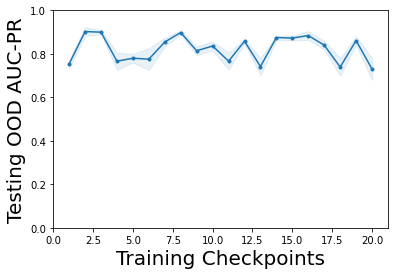}
    \end{subfigure}
    \begin{subfigure}{.245\textwidth}
        \includegraphics[width=\textwidth]{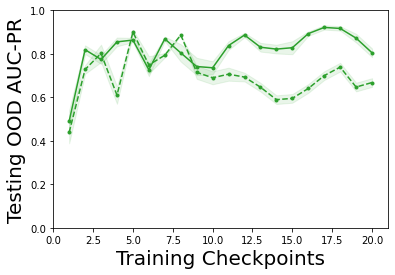}
    \end{subfigure}
        \vspace{-3mm}
    \caption{Comparison of the OOD performance on CartPole. The four uncertainty methods use the sampling-aleatoric or sampling-epistemic strategies at both training and testing time. Ideally, an uncertainty aware-model should achieve high testing reward and high OOD AUC-PR detection score.}
    \label{fig:strategy-testing-ood-auc-pr-performance-cartpole}
        \vspace{-3mm}
\end{figure}
\begin{figure}
    \centering
        \vspace{-3mm}
    \begin{subfigure}{.45\textwidth}
        \includegraphics[width=\textwidth]{resources/sampling-legend.png}
    \end{subfigure}
    \vspace{-3mm}
    
    \begin{subfigure}{.245\textwidth}
        \includegraphics[width=\textwidth]{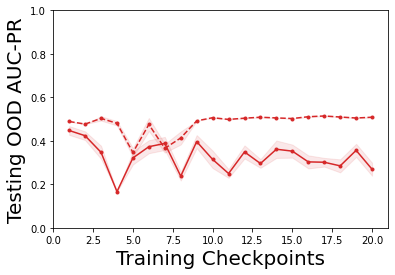}
    \end{subfigure}
    \begin{subfigure}{.245\textwidth}
        \includegraphics[width=\textwidth]{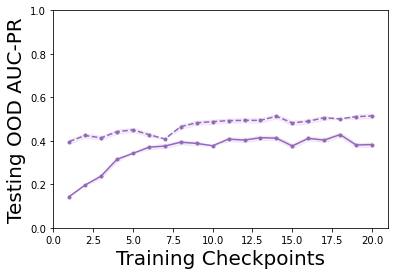}
    \end{subfigure}
    \begin{subfigure}{.245\textwidth}
        \includegraphics[width=\textwidth]{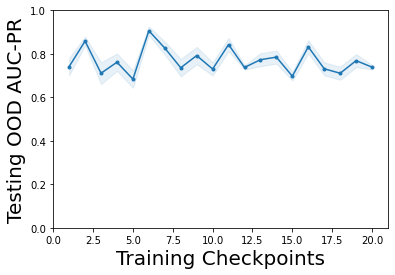}
    \end{subfigure}
    \begin{subfigure}{.245\textwidth}
        \includegraphics[width=\textwidth]{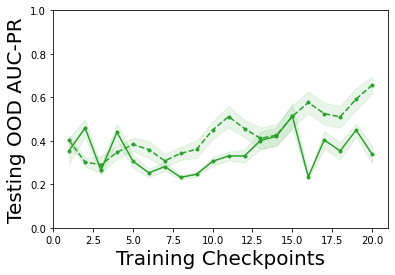}
    \end{subfigure}
        \vspace{-3mm}
    \caption{Comparison of the OOD performance on Acrobot. The four uncertainty methods use the sampling-aleatoric or sampling-epistemic strategies at both training and testing time. Ideally, an uncertainty aware-model should achieve high testing reward and high OOD AUC-PR detection score.}
    \label{fig:strategy-testing-ood-auc-pr-performance-acrobot}
        \vspace{-3mm}
\end{figure}
\begin{figure}
    \centering
        \vspace{-3mm}
    \begin{subfigure}{.45\textwidth}
        \includegraphics[width=\textwidth]{resources/sampling-legend.png}
    \end{subfigure}
    \vspace{-3mm}
    
    \begin{subfigure}{.245\textwidth}
        \includegraphics[width=\textwidth]{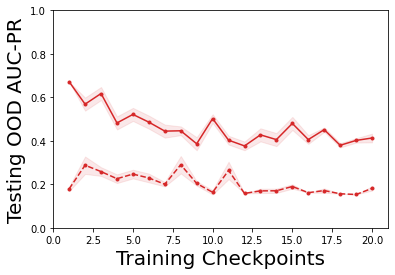}
    \end{subfigure}
    \begin{subfigure}{.245\textwidth}
        \includegraphics[width=\textwidth]{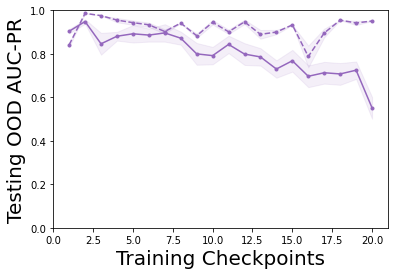}
    \end{subfigure}
    \begin{subfigure}{.245\textwidth}
        \includegraphics[width=\textwidth]{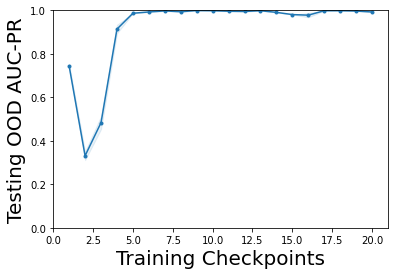}
    \end{subfigure}
    \begin{subfigure}{.245\textwidth}
        \includegraphics[width=\textwidth]{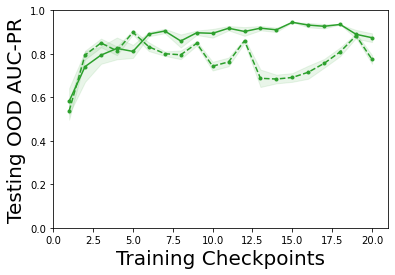}
    \end{subfigure}
        \vspace{-3mm}
    \caption{Comparison of the OOD performance on LunarLander. The four uncertainty methods use the sampling-aleatoric or sampling-epistemic strategies at both training and testing time. Ideally, an uncertainty aware-model should achieve high testing reward and high OOD AUC-PR detection score.}
    \label{fig:strategy-testing-ood-auc-pr-performance-lunarlander}
        \vspace{-3mm}
\end{figure}

\subsection{Comparison with Vanilla DQN}

We show additional results in fig.~\ref{fig:camprison-vanilla-cartpole}, fig.~\ref{fig:camprison-vanilla-acrobot}, fig.~\ref{fig:camprison-vanilla-lunarlander} to compare the sample efficiency and the generalization capacity of the uncertainty models with the vanilla DQN. The vanilla DQN is not eqquiped by default with uncertainty estimates. Therefore, it cannot be used for uncertainty tasks like OOD detection. For the sake of comparison, all models use the epsilon-greedy strategy. We observe that the vanilla DQN achieve significantly lower sample efficiency on CartPole. Further, it achieves less stable generalization performance on LunarLander. In contrast, the four uncertainty methods achieve higher generalization performance especially when using the sampling epistemic strategy (see fig.~\ref{fig:strategy-testing-performance-cartpole}, fig.~\ref{fig:strategy-testing-performance-acrobot} and fig.~\ref{fig:strategy-testing-performance-lunarlander}). These results underline the benefit of predicting and disentangling the aleatoric and the epistemic uncertainty for better sample efficiency and generalization performance.

\begin{figure}
    \centering
    \begin{subfigure}{.7\textwidth}
        \includegraphics[width=\textwidth]{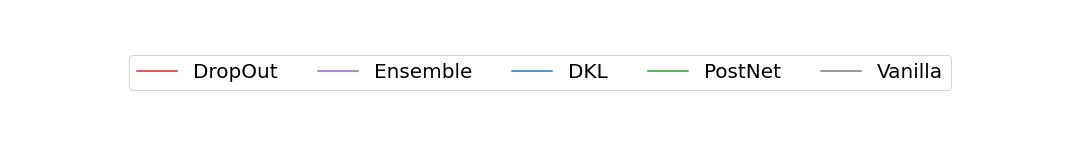}
    \end{subfigure}
    \vspace{-3mm}
    
    \begin{subfigure}{.3\textwidth}
        \includegraphics[width=\textwidth]{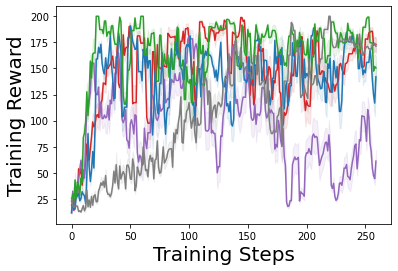}
    \end{subfigure}
    \begin{subfigure}{.3\textwidth}
        \includegraphics[width=\textwidth]{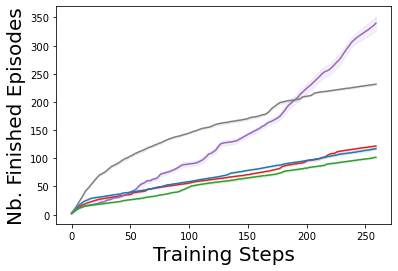}
    \end{subfigure}
    \begin{subfigure}{.3\textwidth}
        \includegraphics[width=\textwidth]{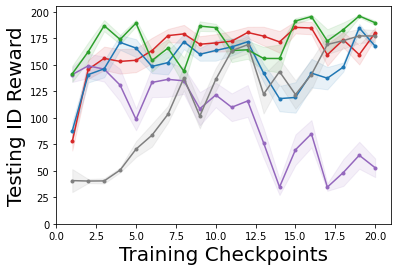}
    \end{subfigure}

        \caption{Comparison of the vanilla DQN with the four uncertainty methods performance on Acrobot. All methods use the epsilon-greedy strategy. The vanilla DQN cannot be evaluated on uncertainty tasks.}
    \label{fig:camprison-vanilla-cartpole}
\end{figure}
\begin{figure}
    \centering
    \begin{subfigure}{.7\textwidth}
        \includegraphics[width=\textwidth]{resources/legend-standard.png}
    \end{subfigure}
    \vspace{-3mm}
    
    \begin{subfigure}{.3\textwidth}
        \includegraphics[width=\textwidth]{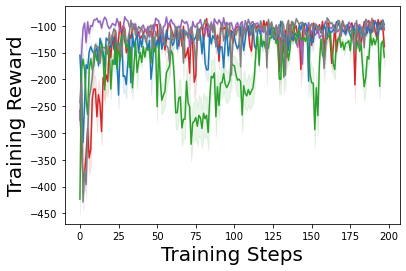}
    \end{subfigure}
    \begin{subfigure}{.3\textwidth}
        \includegraphics[width=\textwidth]{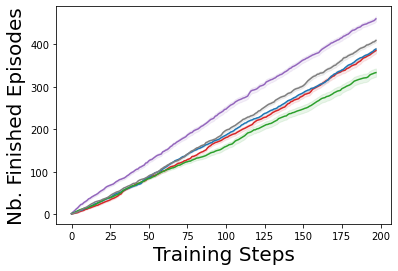}
    \end{subfigure}
    \begin{subfigure}{.3\textwidth}
        \includegraphics[width=\textwidth]{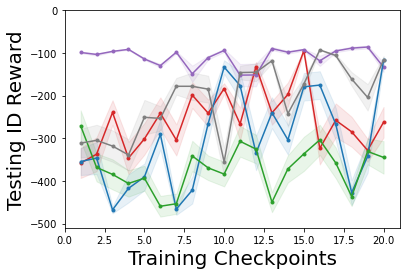}
    \end{subfigure}

        \caption{Comparison of the vanilla DQN with the four uncertainty methods performance on Acrobot. All methods use the epsilon-greedy strategy. The vanilla DQN cannot be evaluated on uncertainty tasks.}
    \label{fig:camprison-vanilla-acrobot}
\end{figure}
\begin{figure}
    \centering
    \begin{subfigure}{.7\textwidth}
        \includegraphics[width=\textwidth]{resources/legend-standard.png}
    \end{subfigure}
    \vspace{-3mm}
    
    \begin{subfigure}{.3\textwidth}
        \includegraphics[width=\textwidth]{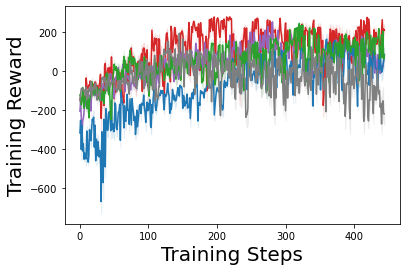}
    \end{subfigure}
    \begin{subfigure}{.3\textwidth}
        \includegraphics[width=\textwidth]{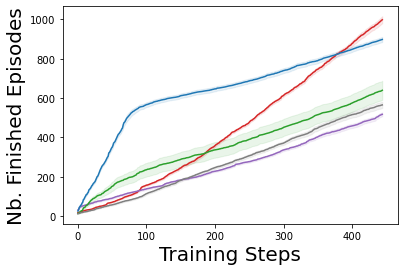}
    \end{subfigure}
    \begin{subfigure}{.3\textwidth}
        \includegraphics[width=\textwidth]{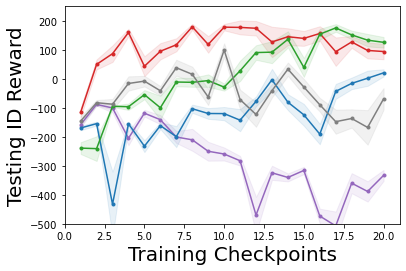}
    \end{subfigure}

        \caption{Comparison of the vanilla DQN with the four uncertainty methods performance on LunarLander. All methods use the epsilon-greedy strategy. The vanilla DQN cannot be evaluated on uncertainty tasks.}
    \label{fig:camprison-vanilla-lunarlander}
\end{figure}

\subsection{Hyperparameter Selection}
\label{app:hyper-parameter-study}

In this section, we present a hyperparameter study for each uncertainty method on the CartPole environment To this end, plot the testing reward and the OOD scores when varying the most important hyper-parameters. we at testing time. We show the hyper-parameter study for dropout when varying the number of samples $n$ and the dropout probability $p$ in fig.~\ref{fig:hyperparameter-dropout-cartpole}. We observe that a higher number of samples achieves a slightly better OOD detection score. Dropout is pretty insensitive to the dropout probability. We show the hyper-parameter study for ensemble when varying the number of networks $n$ in fig.~\ref{fig:hyperparameter-ensemble-cartpole}. While a higher number of networks is supposed to give higher prediction quality \cite{ensembles}, Ensemble looks to give similar results for all number of networks. We show the hyper-parameter study for DKL when varying the number of inducing points $n$, the latent dimension $H$, the kernel type and the batch norm layer in fig.~\ref{fig:hyperparameter-dkl-cartpole}. The batch norm layer appears to improve the results similarly to \cite{postnet}. It facilitates the match between the latent positions output by the encoder and the inducing points. The other hyperparameters consistently show good performances. We show the hyperparameter study for PostNet in fig.~\ref{fig:hyperparameter-postnet-cartpole}. Again, the batch norm layer appears to improve the result stability as observed in \cite{postnet}. It facilitates the match between the latent positions output by the encoder and non-zero density regions learned by the normalizing flows. The other hyperparameters consistently show good performances.

\begin{figure}
\centering
    \begin{subfigure}{.45\textwidth}
        \includegraphics[width=\textwidth]{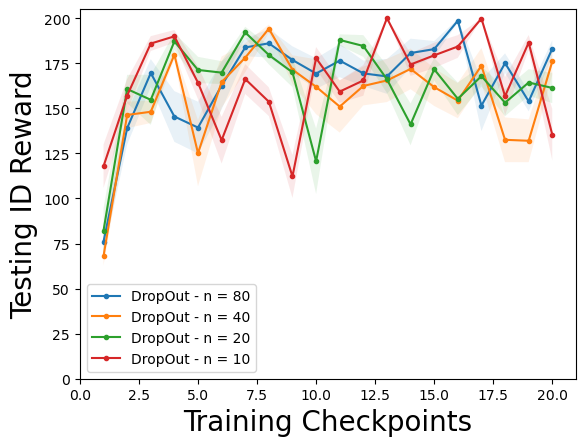}  
    \end{subfigure}
        \begin{subfigure}{.45\textwidth}
        \includegraphics[width=\textwidth]{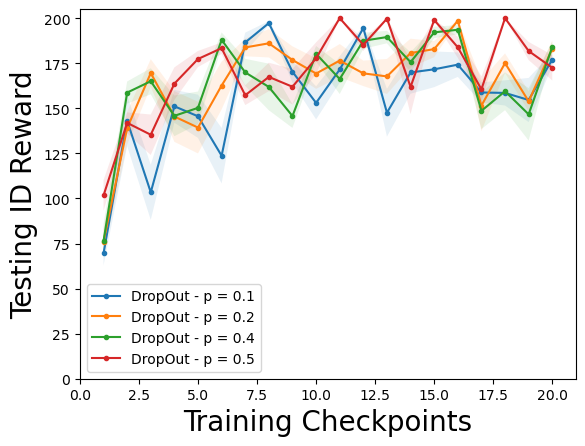}  
    \end{subfigure}
    
    \begin{subfigure}{.45\textwidth}
        \includegraphics[width=\textwidth]{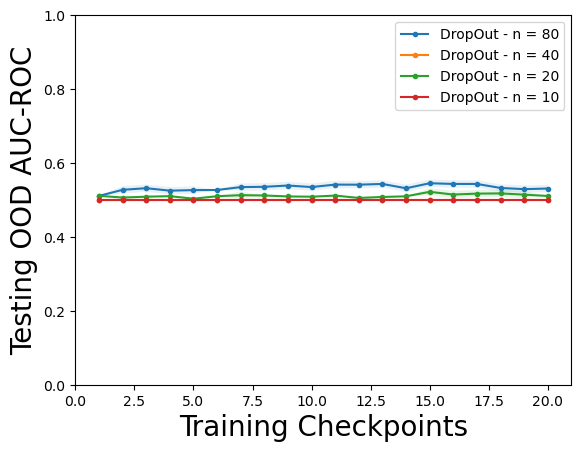}
    \end{subfigure}
        \begin{subfigure}{.45\textwidth}
        \includegraphics[width=\textwidth]{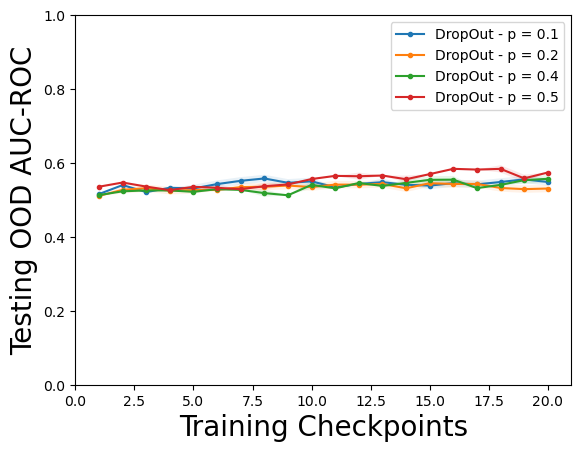}
    \end{subfigure}
        \vspace{-3mm}
    \caption{Hyper-parameter study for DropOut w.r.t. the number of samples $n$ and the dropout probability $p$. Ideally, an uncertainty aware-model should achieve high reward and high OOD detection scores.}
    \label{fig:hyperparameter-dropout-cartpole}
    \vspace{-4mm}
\end{figure}
\begin{figure}
\centering
    \begin{subfigure}{.45\textwidth}
        \includegraphics[width=\textwidth]{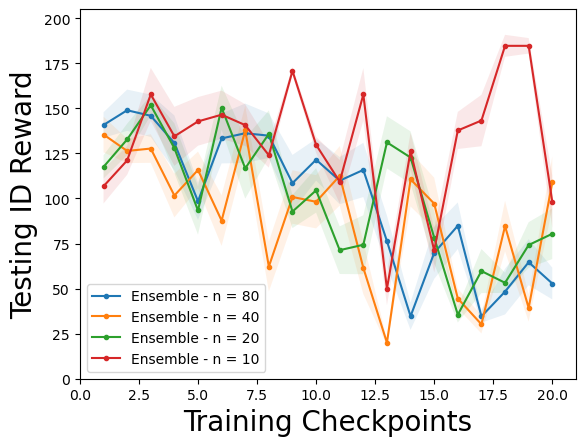}  
    \end{subfigure}
    \begin{subfigure}{.45\textwidth}
        \includegraphics[width=\textwidth]{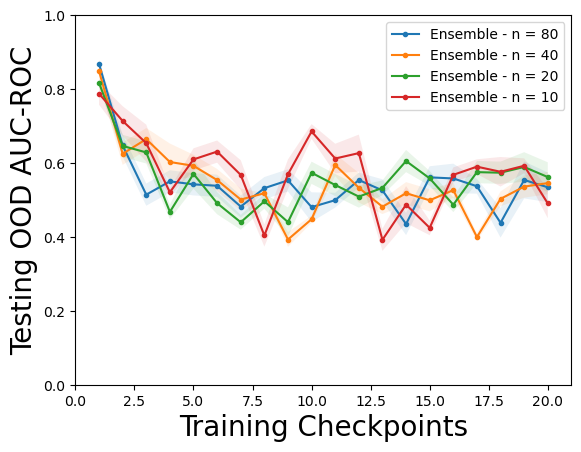}
    \end{subfigure}
        \vspace{-3mm}
    \caption{Hyper-parameter study for Ensemble w.r.t. the number of networks $n$. Ideally, an uncertainty aware-model should achieve high reward and high OOD detection scores.}
    \label{fig:hyperparameter-ensemble-cartpole}
    \vspace{-4mm}
\end{figure}
\begin{figure}
\centering
    \begin{subfigure}{.245\textwidth}
        \includegraphics[width=\textwidth]{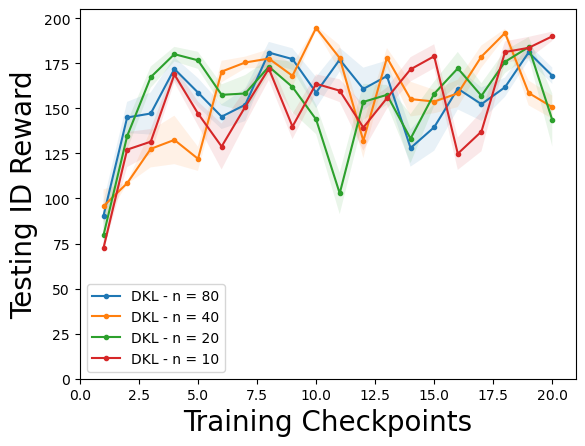}  
    \end{subfigure}
        \begin{subfigure}{.245\textwidth}
        \includegraphics[width=\textwidth]{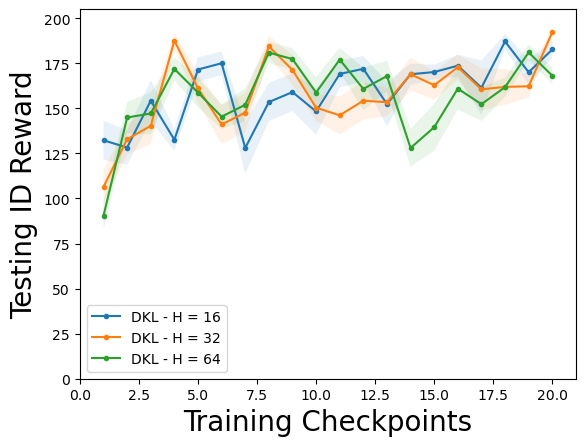}  
    \end{subfigure}
        \begin{subfigure}{.245\textwidth}
        \includegraphics[width=\textwidth]{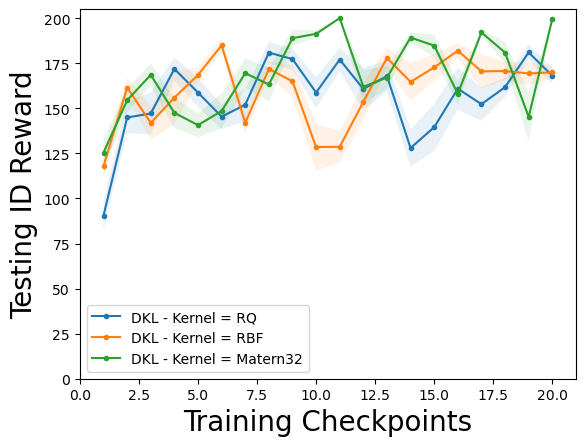}  
    \end{subfigure}
        \begin{subfigure}{.245\textwidth}
        \includegraphics[width=\textwidth]{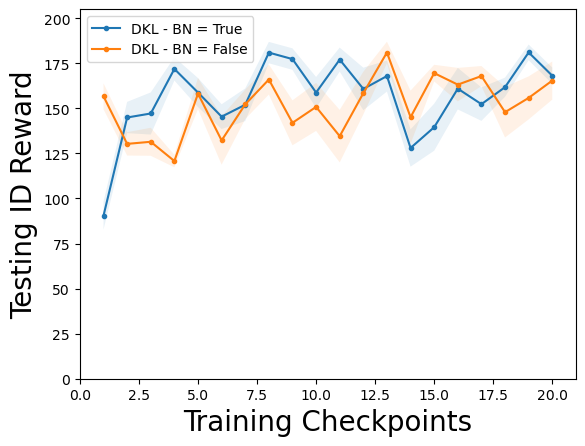}  
    \end{subfigure}
    
    \begin{subfigure}{.245\textwidth}
        \includegraphics[width=\textwidth]{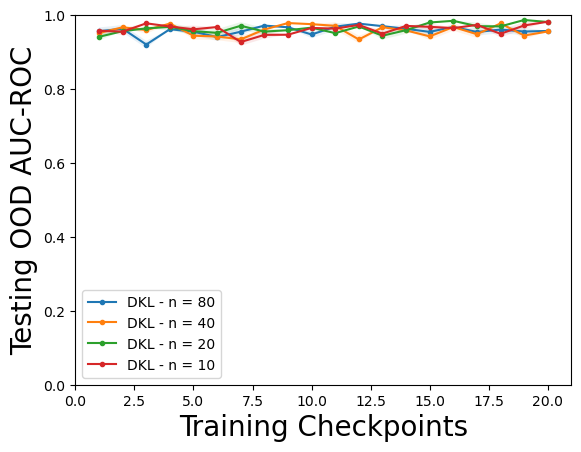}
    \end{subfigure}
        \begin{subfigure}{.245\textwidth}
        \includegraphics[width=\textwidth]{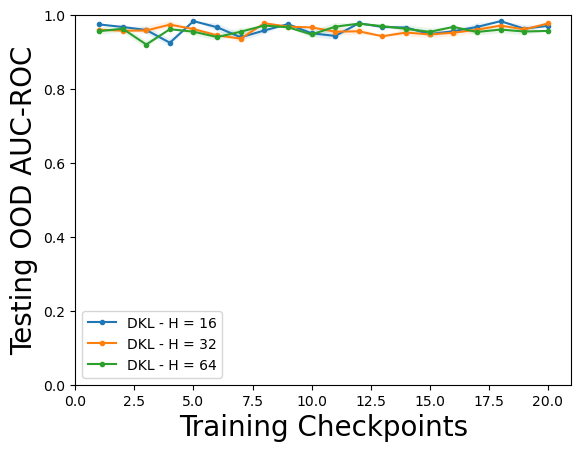}
    \end{subfigure}
    \begin{subfigure}{.245\textwidth}
        \includegraphics[width=\textwidth]{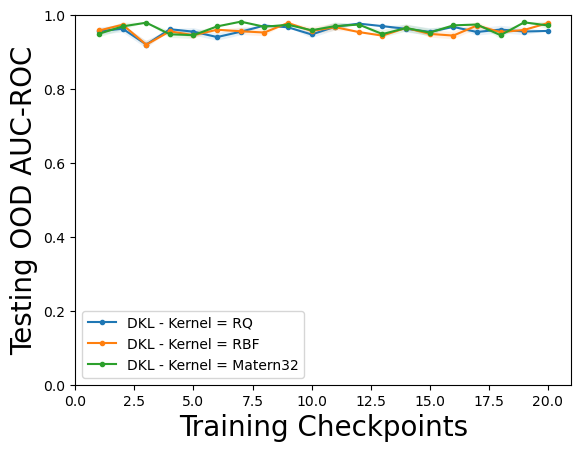}
    \end{subfigure}
        \begin{subfigure}{.245\textwidth}
        \includegraphics[width=\textwidth]{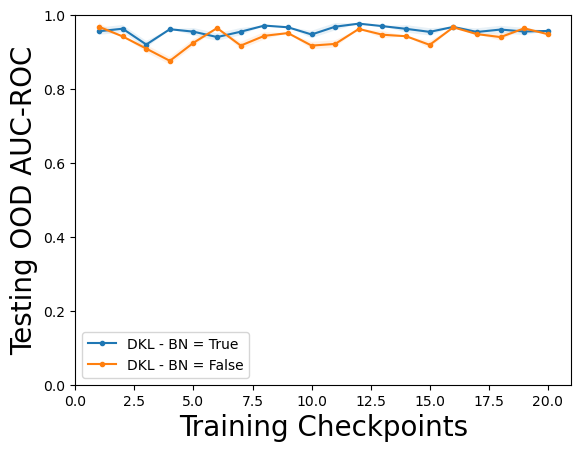}
    \end{subfigure}        
    \vspace{-3mm}
    \caption{Hyper-parameter study for DKL w.r.t. the number of inducing points $n$, the latent dimension $H$, the kernel type and the batch norm layer. Ideally, an uncertainty aware-model should achieve high reward and high OOD detection scores.}
    \label{fig:hyperparameter-dkl-cartpole}
    \vspace{-4mm}
\end{figure}
\begin{figure}
\centering
    \begin{subfigure}{.3\textwidth}
        \includegraphics[width=\textwidth]{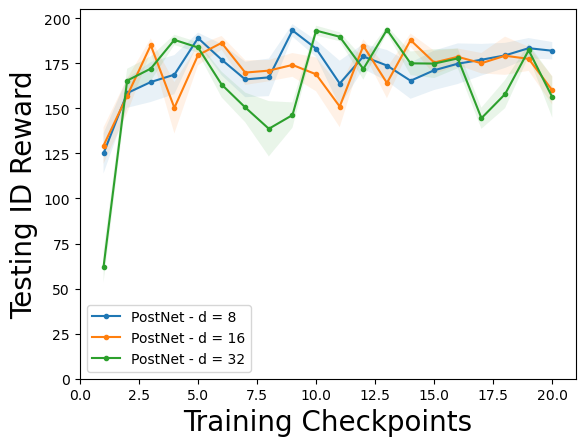}  
    \end{subfigure}
        \begin{subfigure}{.3\textwidth}
        \includegraphics[width=\textwidth]{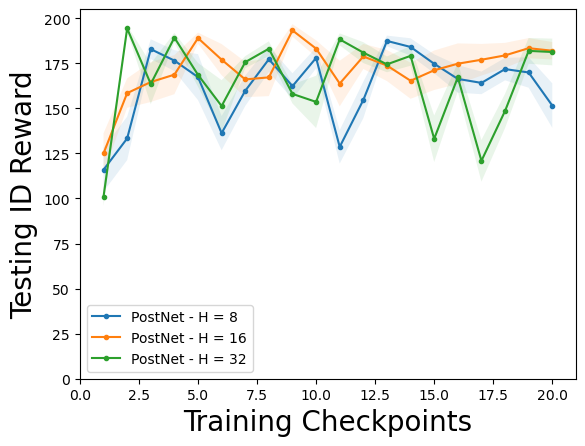}  
    \end{subfigure}
        \begin{subfigure}{.3\textwidth}
        \includegraphics[width=\textwidth]{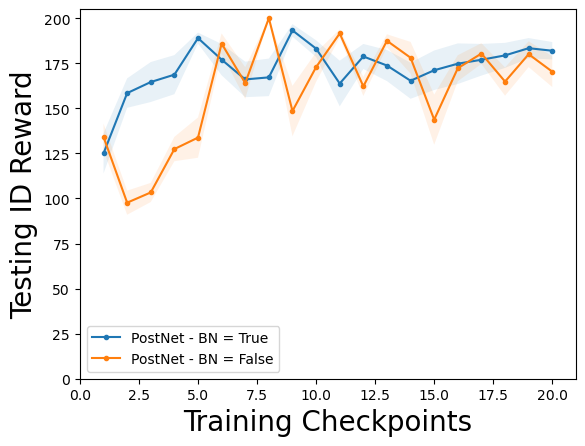}  
    \end{subfigure}
    
    \begin{subfigure}{.3\textwidth}
        \includegraphics[width=\textwidth]{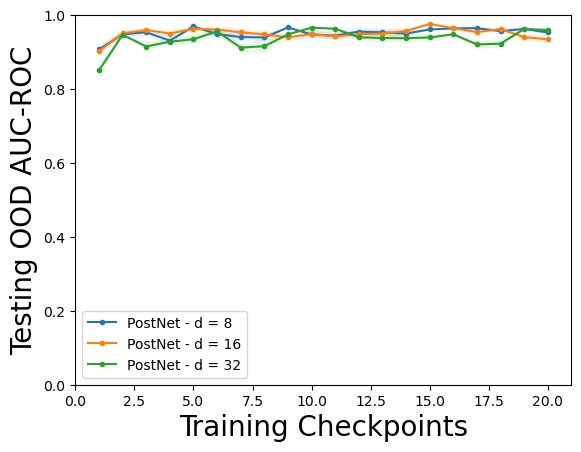}
    \end{subfigure}
    \begin{subfigure}{.3\textwidth}
        \includegraphics[width=\textwidth]{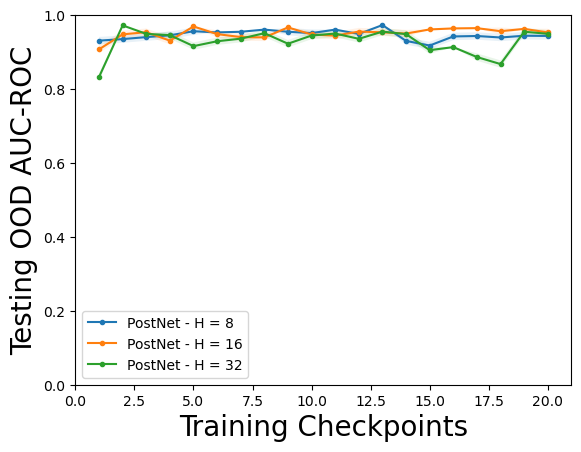}
    \end{subfigure}
        \begin{subfigure}{.3\textwidth}
        \includegraphics[width=\textwidth]{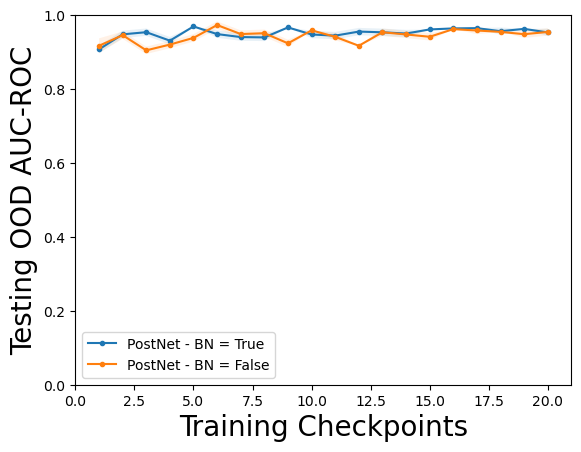}
    \end{subfigure}        
    \vspace{-3mm}
    \caption{Hyper-parameter study for PostNet w.r.t. the flow depth $d$, the latent dimension $H$ and the the batch norm layer. Ideally, an uncertainty aware-model should achieve high reward and high OOD detection scores.}
    \label{fig:hyperparameter-postnet-cartpole}
    \vspace{-4mm}
\end{figure}

\end{document}